%% file: main.tex
\documentclass[11pt]{article} 

\usepackage{microtype}
\usepackage{graphicx}
\usepackage{booktabs} 
\usepackage{tablefootnote}
\usepackage[margin=1in]{geometry}

\usepackage{algorithm}
\usepackage{amsmath}

\usepackage[noend]{algpseudocode} 
\restylefloat{algorithm}
\usepackage[most]{tcolorbox}
\usepackage{hyperref} 
\hypersetup{
    colorlinks,
    linkcolor={blue!100!black!70},
    citecolor={green!50!black!90},
    urlcolor={blue!75!black}
}

\usepackage{amsthm}
\usepackage{epstopdf}
\graphicspath{{Figures/}}
\usepackage{wrapfig}
\usepackage{enumitem}
\algrenewcommand\algorithmicrequire{\textbf{Input:}}
\algrenewcommand\algorithmicensure{\textbf{Output:}}  

\usepackage{color-edits} 
\usepackage{nicefrac} 
\usepackage{authblk}

\input{Config/defs}

\renewcommand{\epsilon}{\varepsilon}
\renewcommand{\eps}{\varepsilon}

\makeatletter
\def\@maketitle{%
  \newpage
  \begin{center}%
  \let \footnote \thanks
    {\LARGE \bf \@title \par}%
  \end{center}%
  \par
  \vskip 0.5em}
\makeatother
\title{The Gaussian Mixing Mechanism: \\R\'enyi Differential Privacy via Gaussian Sketches}

\date{} % This will remove the date

\usepackage[authoryear]{natbib}

\begin{document}
\maketitle
%% TABULAR AUTHOR LIST
\begin{center}
{\large
\begin{tabular}{ccc}
    Omri Lev\(^\dagger\) & Vishwak Srinivasan\(^{\dagger, \ddagger}\) & Moshe Shenfeld\(^{\star, \ddagger}\) \\
    Katrina Ligett\(^\star\) & Ayush Sekhari\(^{\blacklozenge}\) & Ashia C. Wilson\(^\dagger\) \\ 
\end{tabular} 
\vskip 0.5em 
\normalsize
\begin{tabular}{c}
\({}^{\dagger}\)Massachusetts Institute of Technology
\\ 
[0.25em] 
\({}^{\star}\)The Hebrew University of Jerusalem \\
[0.25em] 
\(^\blacklozenge\)Boston University  \\
[0.25em]
\end{tabular}
}
\end{center}

\maketitle 

\begin{abstract}
    Gaussian sketching, which consists of pre-multiplying the data with a random Gaussian matrix, is a widely used technique for multiple problems in data science and machine learning, with applications spanning computationally efficient optimization, coded computing, and federated learning. This operation also provides differential privacy guarantees due to its inherent randomness. In this work, we revisit this operation through the lens of \Renyi Differential Privacy (RDP), providing a refined privacy analysis that yields significantly tighter bounds than prior results. We then demonstrate how this improved analysis leads to performance improvement in different linear regression settings, establishing theoretical utility guarantees. Empirically, our methods improve performance across multiple datasets and, in several cases, reduce runtime. 
\end{abstract} 
%%%%%%%%%%%%%%%%%%%%%%%%%%%%%%%%%%%%%%%%%%%%%%%%%%%%%%%%%%%%%%%%

\begingroup
  \renewcommand\thefootnote{}%
  \footnotetext{Correspondence: \texttt{omrilev@mit.edu}.%
    \ \ Code available at \url{https://github.com/omrilev1/GaussMix}.}%

  \renewcommand\thefootnote{$\ddagger$}%
  \footnotetext{Equal contribution.}%
\endgroup
%%%%%%%%%%%%%%%%%%%%%%%%%%%%%%%%%%%%%%%%%%%%%%%%%%%%%%%%%%%%%%%%%%%%%%%%%%%%%%%%%%%%%%%%%%%%%%%%
\section{Introduction}
\label{s:intro}

Continued advances in data collection and storage have enabled the creation of massive datasets, which have necessitated the design of scalable algorithms for efficient processing, analysis, and inference.  
With the formation of such large datasets, there has also been an increased focus on privacy-preserving analyses in a bid to protect private attributes \citep{apple2017learning, facebook2020protecting, snap2022differential, dpfy_ml}. 
For instance, census data consists of a variety of private information about individuals recorded, and the United States Census Bureau has been taking various measures to ensure confidentiality in the data, and has adopted more modern methods since the late 1900s \citep{census}. 

A popular mathematical notion of privacy is given by \emph{Differential Privacy} (DP) \citep{Dwork_DP_Orig,mcsherry_talwar_2007mechanism, Dwork_AlgFoundationd_DP}, which is currently the de facto standard for privacy-preserving mechanisms. 
A private mechanism is a (randomized) function that takes in a dataset \(X\) and returns a \(\pM(X)\).
The key challenge in designing a ``good'' private mechanism is in balancing the tradeoff between the loss in utility with the increase in privacy-preserving nature of the algorithm; it is possible to obtain complete privacy by simply returning noise regardless of the input, but this clearly does not have good utility.

With regard to designing scalable algorithms to handle massive datasets, \emph{sketching} as a technique \citep{sarlos2006improved} has led to more computationally scalable algorithms relative to na\"{\i}vely working with the full dataset; see \citet{woodruff2014sketching} for a review.
Given a data matrix \(X \in \reals^{n \times d}\), a \emph{sketch} constructs a compressed representation \(\mathsf{S} X \in \reals^{k \times d}\) with \(k \ll n\), where \(\mathsf{S} \in \reals^{k \times n}\) is a random matrix.
For the case where $\mathsf{S}$ is comprised of \iid\ Gaussian elements and sufficiently large \(k\), $\frac{1}{k}\mathsf{S}^{\top}\mathsf{S} \approx \brI_n$ with high probability.
This property is particularly valuable in machine learning applications that rely on inner-products of the form \(A^{\top}B\) for matrices \(A\) and \(B\), as one could now apply a Gaussian sketching matrix \(\mathsf{S}\) to both \(A\) and \(B\) and largely preserve this Hilbert-Schmidt inner product (up to scaling constants).
In this work, we study a mechanism based on a ``noisy'' sketching operation, which we call the \emph{Gaussian Mixing Mechanism} (abbrev. \textsf{GaussMix}) and is defined by
\begin{equation*}
\label{eq:mechanism_intro}
\pM(X) := \mathsf{S}X + \sigma \xi~, \quad \text{with} \quad \mathsf{S} \sim \pN(0, \brI_{k \times n}),~ \xi \sim \pN(0, \brI_{k \times d})~.\tag{\textsf{GaussMix}} 
\end{equation*}
The additive Gaussian noise (hence the ``noisy'' sketching operation) essentially contributes a constant bias to the inner product between the outputs of \ref{eq:mechanism_intro} applied to \(A\) and \(B\) since for sufficiently large \(k\), appealing to standard concentration inequalities, we have 
\begin{align}
    \label{eq:intuition_utility}
    \frac{1}{k}(\mathsf{S}A + \sigma \xi)^{\top}(\mathsf{S}B + \sigma\xi) \approx A^{\top}B + \sigma^2 \brI_d.
\end{align}
Thus, this noisy sketch is potentially well-suited for applications involving inner products.

The first evidence of sketching yielding a differentially private mechanism was given by  \citet{blocki2012johnson}, and later works \citep{sheffet2017differentially,sheffet2019old} focused on establishing that \ref{eq:mechanism_intro} also provides DP. 
While it might seem surprising that \ref{eq:mechanism_intro} is a differentially private mechanism, the operation can be viewed as instantiating the classical \emph{Gaussian mechanism} for DP \citep[Appendix A]{Dwork_AlgFoundationd_DP} for the Gaussian sketched data matrix \(\mathsf{S} X\). 

The intriguing aspect of \ref{eq:mechanism_intro} is that the Gaussian sketching operation contributes to stronger privacy guarantees when compared to the standard Gaussian mechanism.

\paragraph{Intuition for \ref{eq:mechanism_intro}.} 
Interestingly, certain key operations in augmenting mechanisms for improved differential privacy can be expressed as left-multiplication operations.
Two such instances are (1) permuting the rows of the data matrix \(X\) randomly \citep{erlingsson2019amplification}, which can be expressed as left-multiplication of \(X\) by a permutation matrix, and (2) subsampling the data matrix at random, which can be expressed as left-multiplication of \(X\) by a random matrix whose entries are either \(0\) or \(1\) \citep{privacy_subsampling_kasiviswanathan2011can,amplification_subsampling_balle}.
A rough intuition for why sketching is a reasonable operation for preserving privacy is that \(\mathsf{S} X\) generates random linear combinations of rows, which can potentially hide the contribution of any one row. 
Differential privacy---a more rigorous notion---seeks to ensure that the presence of any one particular row in the data can not be guessed well, even by an adversary that has knowledge of the other rows.
Note that sketching by itself is not impervious to such a setting: consider a data matrix \(X\) where all but one row contain 0s, and in this case, it is possible to make educated guesses for what the non-zero row might be from the sketched data matrix \(\mathsf{S} X\).
This motivates the inclusion of the noise addition in \ref{eq:mechanism_intro}, which also covers this adversarial situation.
Informally speaking, through our analysis, we observe that the ``richness'' of the data matrix (quantified by the minimum singular value) contributes to the success of the sketching mechanism as a differentially private mechanism.

~\\ 
\textbf{Contributions.~~} Our key contributions are: 
\begin{enumerate}[leftmargin=*, label=\(\bullet\)] 
      \item \textbf{Tighter Privacy Analysis via R\'enyi Differential Privacy (RDP).} We present a new RDP analysis of \ref{eq:mechanism_intro}, which, to the best of our knowledge, has not been previously explored. Our bounds are simpler to derive than existing analyses and, in several cases, are tight with respect to RDP. Notably, our results show that \ref{eq:mechanism_intro} achieves stronger privacy guarantees than those established by \citet{sheffet2019old} for the same mechanism parameters. These improved guarantees extend to settings that already rely on \ref{eq:mechanism_intro}, such as \citet{Coded_Yona, dp_coded_regression, CodedFederated_MIDP, bartan2023distributed}.    
    
    \item \textbf{Algorithm for Private Ordinary Least Squares.}
    We use \ref{eq:mechanism_intro} to develop a differentially private algorithm for linear regression under bounded data, extending the framework of \citet{sheffet2017differentially}. Leveraging our new RDP analysis, we derive bounds on the excess empirical risk that, in some cases, match those of the AdaSSP algorithm by \citet{wang_adassp}, a standard benchmark under domain-bound assumptions. We validate these theoretical results through empirical evaluations, where our method consistently outperforms the baselines of \citet{sheffet2017differentially} and \citet{wang_adassp} across several benchmark datasets. Furthermore, we show that retrofitting the improved RDP analysis into the algorithm of \citet{sheffet2017differentially} yields improved results, further demonstrating the broad applicability of our improved analysis.

    \item \textbf{Algorithm for Private Logistic Regression.}
    We adapt the algorithm for private linear regression to perform differentially private logistic regression. This specifically works by using a second-order approximation of the loss function~\citep{huggins2017pass, ferrando2024private} that reduces the problem to a differentially private quadratic minimization, which makes it amenable to apply \ref{eq:mechanism_intro}. We derive theoretical guarantees for our proposed method and give numerical simulations over certain datasets that demonstrate improvements over the commonly used objective perturbation method~\citep{chaudhuri2011differentially} in both accuracy and computation time. 
\end{enumerate} 

\paragraph{Organization of the Paper.}
In \sectionref{s:related_works}, we discuss related work. \sectionref{s:prelims} introduces necessary preliminaries. We re-introduce the Gaussian mixing mechanism together with our new privacy analysis in \sectionref{s:mechanism}, including comparisons with prior bounds. Applications of our mechanism to DP linear regression (\sectionref{s:linear}) and DP logistic regression (\sectionref{s:logistic}) are given in \sectionref{s:apps}, supported by theoretical and empirical evaluations.

%%%%%%%%%%%%%%%%%%%%%%%%%%%%%%%%%%%%%%%%%%%%%%%%%%%%%%%%%%%%%%
\section{Related Work} 
\label{s:related_works}

A substantial body of research has investigated the use of random matrix projections for privacy, particularly through the Johnson–Lindenstrauss (JL) transform and its variants~\citep{blocki2012johnson, kenthapadi2012privacy_via_JL, sheffet2017differentially, showkatbakhsh2018privacy, sheffet2019old}.
 \citet{sheffet2015private, sheffet2017differentially, sheffet2019old} and \citet{showkatbakhsh2018privacy} are the most relevant to our work, as they propose using Gaussian sketches for private linear regression.
\citet{sheffet2015private,sheffet2017differentially,sheffet2019old} analyze the privacy-preserving characteristics of \ref{eq:mechanism_intro} for this problem, and show that it achieves \((\varepsilon,\delta)\)-differential privacy for certain settings of \(\sigma, k\).
\citet{showkatbakhsh2018privacy}  study the same mechanism but under a modified notion of differential privacy known as MI-DP (defined by \citet{cuff2016differential}).

The mathematical intuition for why sketching is useful in privacy-sensitive optimization was studied by \citet{pilanci2015randomized}. 
They observe that the mutual information between $X$ and its sketched version $\mathsf{S}X$ can not be too large, thus providing a form of privacy.
However, their analysis is centered around information-theoretic principles and quantities such as the mutual information, and does not actually provide guarantees of differential privacy.
More recently, \citet{bartan2023distributed} provides another application of Gaussian sketching in the specific context of DP distributed linear regression, based on the results obtained by \citet{sheffet2015private}. 

In contrast to these prior approaches, our work uses the stronger RDP framework of \citet{mironov2017renyi} in its privacy analysis. We derive a simple, closed-form expression for the RDP curve $\varepsilon(\alpha)$ corresponding to \ref{eq:mechanism_intro}, which is tight and improves upon the bounds obtained in earlier works. These improvements yield practical gains in the performance of our algorithm for DP linear regression relative to both \citet{sheffet2019old} and other common alternatives such as \citet{wang_adassp}, and can be further used in other settings that currently use \eqref{eq:mechanism_intro}, such as \citet{Coded_Yona, dp_coded_regression, CodedFederated_MIDP, bartan2023distributed}. Moreover, we further derive a computationally efficient algorithm for private logistic regression. Our numerical experiments demonstrate that these enhancements translate into improved accuracy over standard baselines.

%%%%%%%%%%%%%%%%%%%%%%%%%%%%%%%%%%%%%%%%%%%%%%%%%%%%%%%%%%%%%%%%%%%%%%%%%%%%%%%%%%%%%%%%%%%%%%%%
\section{Preliminaries} 
\label{s:prelims}

%%%%%%%%%%%%%%%%%%%%%%%%%%%%%%%%%%%%%%%%%%%%%%%%%%%%%%%%%%%%%%%%%%%%%%%%%%%%%%%%%%%%%%%%%%%%%%%%%%%%%%
\textbf{Basic Notation.} 
We denote random variables in sans-serif (e.g., \(\mathsf{X}, \mathsf{y}\)), and their realizations in serif (e.g., \(X, y\) resp.).
The set \(\{1, \ldots, n\}\) is denoted by \([n]\). 
For a vector \(A \in \reals^{d}\) its Euclidean norm is denoted by \(\norm{A}\).
The all-zeros column vector of length \(d\) is denoted by \(\vec{0}_{d}\).
The $k\times k$ identity matrix is $\brI_k$ and $\pN(0, \brI_{k_1\times k_2})$ denotes a $k_1\times k_2$ matrix of \iid \ standard Gaussian entries. See \appendixref{app:notation} for a detailed discussion of notation.
%%%%%%%%%%%%%%%%%%%%%%%%%%%%%%%%%%%%%%%%%%%%%%%%%%%%%%%%%%%%%%%%%%%%%%%%%%%%%%%%%%%%%%%%%%%%%%%%%%%%%%%%%%%%

\paragraph{Differential Privacy.}
Differential privacy relies on the notion of a ``neighboring'' dataset, which we introduce first. Two datasets $X, X'$ are called \emph{neighbors} if \(X'\) is formed by removing an element from \(X\) \footnote{For simplicity, we identify a removal of a row with its replacement by \(\vec{0}_{d}\), so the dimension remains constant. This notion is sometimes referred to as \emph{zero-out} neighboring.} or vice-versa, and we use \(X \simeq X'\) to denote this relation.
In this work, we focus on datasets that are elements of \(\reals^{n \times d}\) \ie, matrices with \(n\), \(d\)-dimensional real-valued rows.
For \(X\) input to a mechanism, we assume knowledge of an upper bound \(\textsc{C}_{X}\) (called the \emph{row bound}) where \(\|x_{i}\| \leq \textsc{C}_{X}\) for all \(i \in [n]\).
Intuitively, differential privacy, formalized in the next definition, requires that a randomized algorithm induce nearly identical output distributions given neighboring input datasets.

\begin{defn}[{$(\eps,\delta)$-Differential Privacy \citep{dwork2006calibrating}}]
\label{defn:eps_delta_DP}
A randomized mechanism $\pM$ is said to satisfy \emph{$(\eps, \delta)$-differential privacy} if for all \(X, X'\) such that $X' \simeq X$ and measurable subsets $\pS \subseteq \mathrm{Range}(\pM)$,
\begin{align}
    \label{eq:eps_delta_DP}
    \Pr(\pM(X) \in \pS) \leq e^{\eps} \cdot \Pr(\pM(X') \in \pS) + \delta.
\end{align}
\end{defn}

A secondary, somewhat stronger notion of differential privacy that we adopt throughout this work is given by \RenyiDP, first introduced by \citet{mironov2017renyi}.
\begin{defn}[{$(\alpha,\eps(\alpha))$-\RDP \citep{mironov2017renyi}}]
\label{defn:RDP}  
A randomized mechanism $\pM$ is said to satisfy $(\alpha,\eps(\alpha))$-\RDP for some $\alpha > 1$ if for all $X, X'$ such that \(X \simeq X'\),
\begin{align}
    \label{eq:RDP}
    D_{\alpha}\left(\pM(X)\,\|\,\pM(X')\right) \leq \eps(\alpha),
\end{align}
where $D_{\alpha}(P\,\|\,Q) \coloneqq \frac{1}{\alpha - 1} \log\left(\underset{\sx \sim Q}{\mathbb{E}} \left[\left(\frac{P(\sx)}{Q(\sx)}\right)^{\alpha}\right]\right)$ denotes the $\alpha$-\Renyi divergence \citep{renyi1961measures}.
\end{defn}

\((\varepsilon, 0)\)-DP can be viewed as ensuring that the likelihood ratio of events induced by neighboring datasets are uniformly bounded, and \(\delta\) in \((\varepsilon,\delta)\)-DP provides some additive slack on this condition.
On the other hand, for any \(\alpha > 1\), \((\alpha, \varepsilon(\alpha))\)-\RenyiDP can be seen as another control that bounds the moments of this likelihood ratio.

The conversion between $(\alpha, \eps(\alpha))$-RDP and $(\eps,\delta)$-DP is given in the next proposition. 

\begin{prop}[{\citet[Proposition 12]{canonne2020discrete}}]
\label{prop:Renyi_classical_translate}
If $\pM$ satisfies $(\alpha, \eps(\alpha))$-\RDP, then it also satisfies $(\eps_{\mathrm{DP}}, \delta)$-\DP for any $0 < \delta < 1$, where $\eps_{\mathrm{DP}} = \eps(\alpha) + \log\left(1 - \frac{1}{\alpha}\right) - \frac{\log(\alpha\delta)}{(\alpha - 1)}$.  
\end{prop}

Both $(\alpha, \varepsilon(\alpha))$-RDP and \((\varepsilon, \delta)\)-DP satisfy key properties such as graceful degradation under composition and post-processing. In particular, the post-processing property ensures that if a mechanism \(f\) satisfies either privacy definition, then so does \(g\circ f\) for any (possibly randomized) function $g$ ~\citep{Dwork_AlgFoundationd_DP, mironov2017renyi}. 

With the definition of \RDP above, of particular note is a special family of mechanisms that satisfy \((\alpha, \varepsilon(\alpha))\)-RDP for a range of values of \(\alpha\), and such that $\eps(\alpha)$ grows at most linearly in $\alpha$ in this range.
Such mechanisms are said to satisfy truncated concentrated DP, formally defined next.

\begin{defn}[tCDP {\citep{bun2018composable}}]
\label{defn:tcdp}
Let $\rho > 0$ and $w > 1$. A mechanism $\pM$ satisfies $(\rho, w)$-tCDP if $\KLDA{\pM(X)}{\pM(X')} \leq \rho \alpha$ for all neighboring $X,X'$ and for all $\alpha \in (1, w)$.
\end{defn}

The purpose of introducing tCDP is due to its connection to \((\varepsilon, \delta)\)-DP, given below.
\begin{prop}[{\citet[Lemma  6]{bun2018composable}}]
\label{prop:conversion_tcdp_dp}
If \(\pM\) satisfies \((\rho, w)\)-\normalfont{tCDP}, then its also satisfies \((\varepsilon_{\mathrm{DP}}, \delta)\)-DP for all $\delta > 0$ where
\begin{equation*}
    \varepsilon_{\mathrm{DP}} =
    \begin{cases}
        \rho + 2\sqrt{\rho \cdot \log(\nicefrac{1}{\delta})} & \text{if } \log\left(\nicefrac{1}{\delta}\right) \leq (w - 1)^{2} \cdot\rho \\
        \rho \cdot w + \frac{\log(\nicefrac{1}{\delta})}{w - 1} & \text{if } \log\left(\nicefrac{1}{\delta}\right) > (w - 1)^{2} \cdot\rho~.
    \end{cases}
\end{equation*}
\end{prop}

\textbf{Gaussian mechanism.} The \emph{Gaussian mechanism} is a standard baseline for achieving \((\varepsilon, \delta)\)-DP by simply adding Gaussian noise to (some function of) the data before releasing it. In our notation it amounts to $\pM_{\mathrm{G}}(X) = X + \sigma \xi$ with $\xi \sim \pN(0, \brI_{n\times d})$ where $\sigma > 0$ controls the noise magnitude. The Gaussian mechanism is $\left(\frac{\textsc{C}^{2}_{X}}{2\sigma^{2}}, \infty \right)$-tCDP (also known as zCDP \citep{bun2016concentrated}), and $(\eps, \delta)$-DP where $\eps = \frac{\sqrt{2 \log(1.25/\delta)}}{\sigma}$ for any $\delta \in (0,1)$ \citep[Appendix ~A]{Dwork_AlgFoundationd_DP}.  

%%%%%%%%%%%%%%%%%%%%%%%%%%%%%%%%%%%%%%%%%%%%%%%%%%%%%%%%%%%%%%%%%%%%%%%%%%%%%%%%%%%%%%%%%%%%%%%%%%%%%%%%%%%%
\section{R\'enyi-DP Analysis of the Gaussian Mixing Mechanism} 
\label{s:mechanism}
We start by providing a new privacy analysis of \ref{eq:mechanism_intro}, under the assumption that we have a lower bound \(\overline{\lambda}_{\min}\) for the minimum eigenvalue of \(X^{\top}X\) (called the \emph{scale bound}).

\begin{lem}
    \label{lem:Privacy_First}
    For any data matrix \(X \in \reals^{n \times d}\) that satisfies row bound \(\textsc{C}_{X}\) and scale bound \(\overline{\lambda}_{\min}\), \ref{eq:mechanism_intro} with parameters \(k\) and \(\sigma\) such that \(\gamma \coloneqq {\textsc{C}^{-2}_{X}}\cdot (\sigma^{2} + \overline{\lambda}_{\min}) > 1\) satisfies \((\alpha, \varphi(\alpha; k, \gamma))\)-\RDP for any $\alpha \in (1, \gamma)$, where 
    \begin{align}
        \label{eq:privacy_calib}
        \varphi\left(\alpha; k, \zeta\right) \coloneqq \frac{k\alpha}{2(\alpha{-}1)} \log\left(1{-}\frac{1}{\zeta}\right) - \frac{k}{2(\alpha{-}1)} \log\left(1{-}\frac{\alpha}{\zeta}\right).
    \end{align}
\end{lem}
To help understand the role of \(\gamma\), this can be viewed as a lower bound on the minimal eigenvalue of the matrix \(\widetilde{X}^{\top}\widetilde{X}\) where \(\widetilde{X} \coloneqq \textsc{C}_{X}^{-1} \cdot [X^{\top}, \sigma \brI_{d}]^{\top}\).
Using this perspective, the noise addition can be seen as a way to artificially raise the minimum singular value of the matrix $X$ to a predetermined threshold $\gamma$, after which a standard Gaussian sketching step is applied. This reinforces the prior intuition that privacy arises from applying Gaussian sketching to a matrix with a sufficiently large minimum singular value. 

\begin{proof}[Proof sketch of Lemma~\ref{lem:Privacy_First}] 
    For a fixed \(X\), we first deduce that every row of \(\pM(X)\) is distributed according to a multivariate Gaussian distribution with zero mean and covariance $\Sigma \coloneqq X^{\top}X + \sigma^2 \brI_d$.
    Using the closed form expression of the \Renyi divergence between multivariate Gaussian distributions, the quantity \(\KLDA{\pM(X)}{\pM(X')}\) is a monotonic function of \(x^{\top}\Sigma^{-1}x\) where \(x\) is the row in \(X\) that is zeroed out in \(X'\).
    This quantity can be bounded as \(x^{\top}\Sigma^{-1}x \leq \left(\lambda_{\min}(\Sigma)\right)^{-1} \cdot \|x\|^{2}\), and \(\lambda_{\min}(\Sigma) \geq \overline{\lambda}_{\min} + \sigma^{2}\) by the scale bound. 
    We defer full details to \appendixref{app:privacy}. 
\end{proof} 

Following the analysis from \appendixref{app:privacy}, the function $\varphi$ is non-negative for every $\alpha \in (1, \gamma)$ and further upper bounds the \RDP curve of $\pM(X)$. Moreover, for sufficiently large \(\gamma\), the function \(\varphi\) in \lemmaref{lem:Privacy_First} grows at most linearly in \(\alpha\) in a certain range, suggesting that \ref{eq:mechanism_intro} satisfies \definitionref{defn:tcdp}. The next corollary formalizes this observation.
\begin{corol}
    \label{corol:tcdp_proof}
    Consider the setup of \lemmaref{lem:Privacy_First}.
    If \(\gamma > \frac{5}{2}\), then \ref{eq:mechanism_intro} satisfies \((k/2\gamma^{2}, 2\gamma/5)\)-\emph{tCDP}.
\end{corol}

\paragraph{Tightness of Bound.} 
In our computations, the only inequality that occurs is the one stated in the proof sketch.
This inequality is tight when the row in which $X$ differs from $X'$ happens to be the eigenvector corresponding to the minimal eigenvalue of $X^{\top} X$.
This implies that the lemma is tight for specific inputs, ensuring the bound is optimal under certain input assumptions. More precisely, for the set of matrices $X\in \reals^{n\times d}$ such that $n\geq d$ with norm bounds $\textsc{C}_X$ and such that $\lambda_{\min}(X^{\top}X) \geq \textsc{C}^{2}_X$, our bound is achieved with equality for any $X$ such that $X^{\top}X = \textsc{C}^{2}_X \cdot \brI_d$, corresponding to the case where $\frac{1}{\textsc{C}_X}X$ is a semi-orthogonal matrix. 

\paragraph{Comparison to Existing Literature.}  
Combining \corollaryref{corol:tcdp_proof} with the conversion from RDP to DP we get that \ref{eq:mechanism_intro} is $(\eps, \delta)$-DP where $\eps \leq \frac{k}{2\gamma^{2}} + \frac{\sqrt{2 k \log(\nicefrac{1}{\delta})}}{\gamma}$. This bound is comparable to the bound from \citep[Theorem~2]{sheffet2019old} which states that \ref{eq:mechanism_intro} is $(\eps, \delta)$-DP where $\eps \coloneqq \frac{2 \sqrt{2 k \log(4/\delta)}}{\gamma} + \frac{2\log(4/\delta)}{\gamma}$. We note that the leading term in both expressions is similar up to a factor of $2$, and where the bound from  \citep{sheffet2019old} suffers the additional $\frac{2\log(4/\delta)}{\gamma}$ term. Numerical analysis based on the exact RDP guarantee (derived by using the conversion from RDP to DP of \propositionref{prop:Renyi_classical_translate}) provides a larger improvement of \lemmaref{lem:Privacy_First} over the bound of \citep{sheffet2019old}; see \figureref{fig:comparison_Sheffet_ours} for a numerical comparison.  

Moreover, we note that with \ref{eq:mechanism_intro} the entire dataset contributes to the privacy protection of a single element via mixing. Since $\overline{\lambda}_{\min}$ is at most $\frac{n\textsc{C}^{2}_X}{d}$ (see discussion in \citep{wang_adassp}), this resembles privacy amplification by shuffling or subsampling, whereby the added noise to other elements contributes to the privacy protection of any element. The dependence of $\eps$ on the parameters is not trivially comparable since $k$ affects only \ref{eq:mechanism_intro}, and the utility of the two outputs might be arbitrarily different.

\begin{figure}[t]
  \centering
  \begin{subfigure}[c]{0.49\linewidth}
    \centering
    \includegraphics[height=6.6cm]{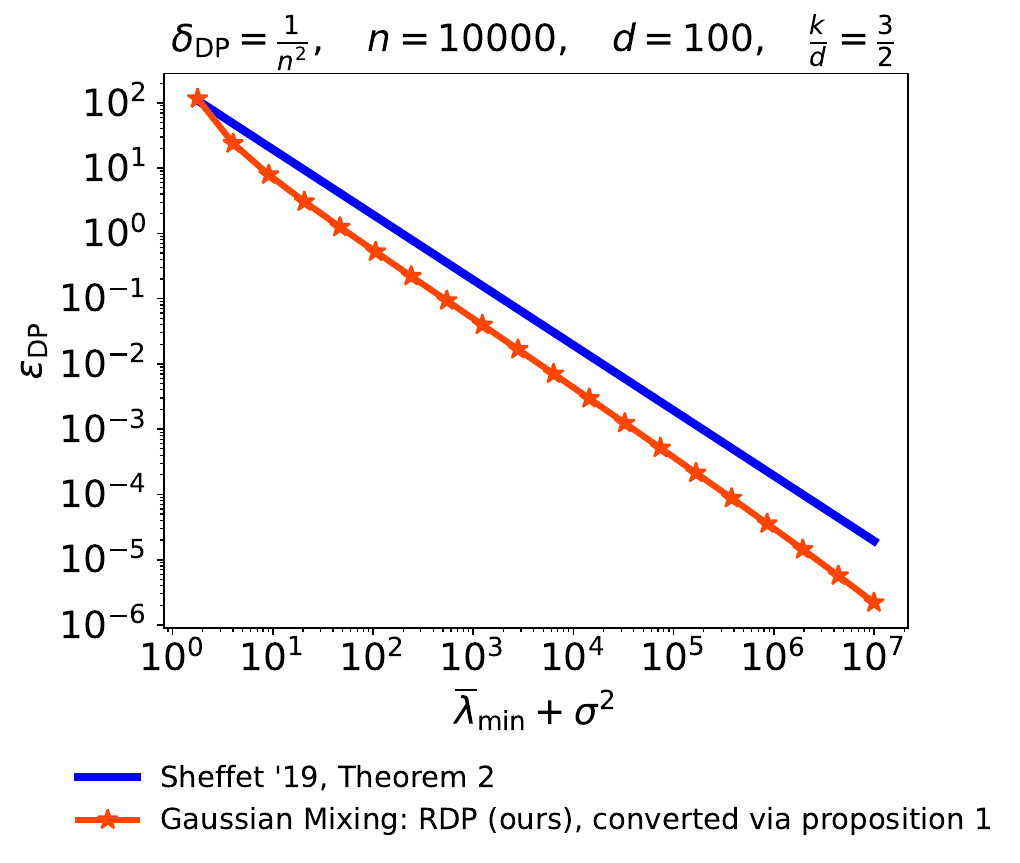}
    \caption{Required $\varepsilon_{\mathrm{DP}}$.}
    \label{fig:comparison_sheffet_ours_main}
  \end{subfigure}
  \hfill
  \begin{subfigure}[c]{0.49\linewidth}
    \centering
    \includegraphics[height=6.6cm]{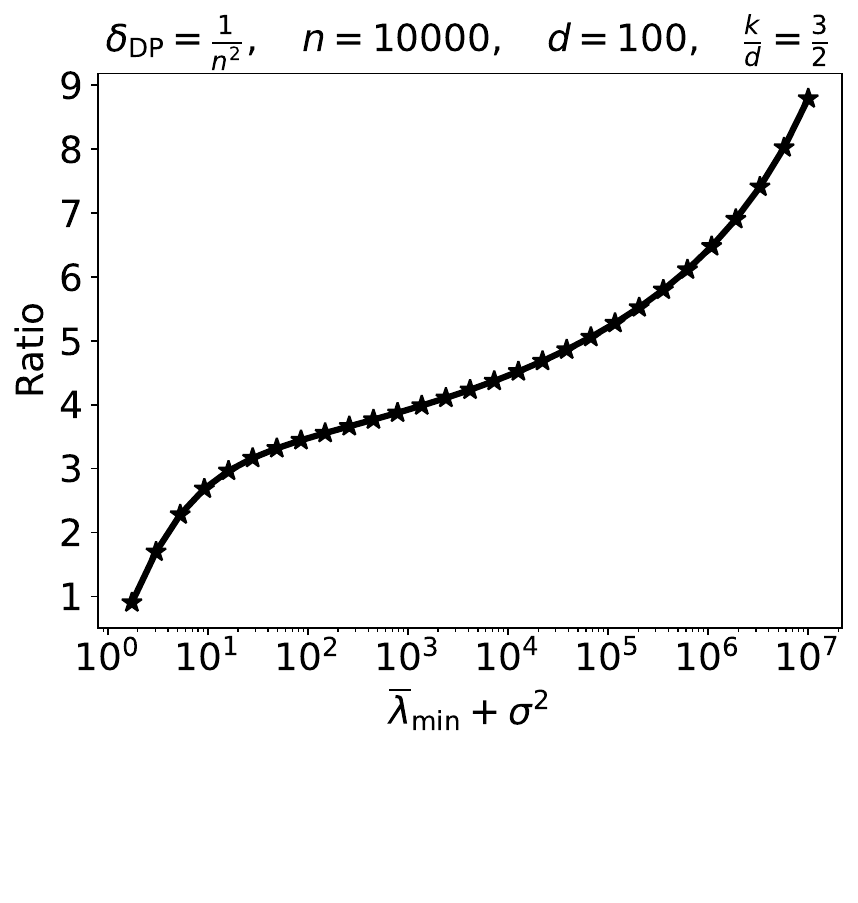}
    \caption{Ratio between the bounds on $\eps_{\mathrm{DP}}$.}
    \label{fig:comparison_sheffet_ours_ratio}
  \end{subfigure}

  \caption{$\overline{\lambda}_{\min} + \sigma^2$ required for attaining $\varepsilon_{\mathrm{DP}}$, using \lemmaref{lem:Privacy_First} converted to $(\varepsilon_{\mathrm{DP}}, \delta)$-DP via \propositionref{prop:Renyi_classical_translate}, and compared to the bound from \citep[Theorem~2]{sheffet2019old}. We plot (a) the relation between $\eps_{\mathrm{DP}}$ and $\overline{\lambda}_{\min} + \sigma^2$ and (b) the ratio between the $\eps_{\mathrm{DP}}$ attained by each bound. Our improved \RDP analysis reduces the final $\varepsilon_{\mathrm{DP}}$, illustrating the benefit of our new approach.}
  \label{fig:comparison_Sheffet_ours}
\end{figure}

\paragraph{Usage Without $\overline{\lambda}_{\min}$.}
Since the bounded scale assumption cannot be enforced in the general case, one way to utilize Lemma \ref{lem:Privacy_First} is by simply using the fact that $\overline{\lambda}_{\min} \ge 0$ and relying solely on the added noise. In this case, the \ref{eq:mechanism_intro} mechanism can be thought of as a complicated counterpart of the Gaussian mechanism. Comparing their privacy guarantees reveals a clear similarity, with $\sqrt{k}/\gamma$ taking the role of $1/\sigma$ in the Gaussian mechanism. In this case $\gamma = \sigma^{2}$ which amounts to a quadratic dependence of $\eps$ on $\sigma$ in the leading term (rather than the linear one in the Gaussian mechanism), balanced by $\sqrt{k}$, which implies improved privacy results when $\sigma > \sqrt{k}$. We note that the output of these two mechanisms follows different distributions (and even dimensions), so the privacy-utility tradeoff depends on the use case.

\paragraph{Instance Specific Bound.} While setting $\overline{\lambda}_{\min} = 0$ provides a privacy guarantee for any matrix, one might wish to reduce the added noise in a data-dependent manner based on $\lambda_{\min}(X^{\top}X)$. Since this quantity is data-dependent, it must be used in a privacy-preserving manner, which is achieved using Algorithm \ref{alg:gauss_mix_full}. A similar algorithm was proposed by \citet[Algorithm 1]{sheffet2017differentially}, with $\sz$ sampled from a Laplace distribution, and $\eta$ set either to $0$ or to $\sqrt{\gamma}$ based on $\widetilde{\lambda}$, rather than interpolating between the two as we do. 

\begin{figure*}[t]
    \centering
    \begin{minipage}[t]{\textwidth}
        \begin{algorithm}[H]
        \caption{\texttt{ModifiedGaussMix}} 
        \begin{algorithmic}[1]
        \Require Dataset $X \in \reals^{n\times d}$, row bound \(\textsc{C}_{X}\), parameters $k, \gamma, \tau, \eta$.
        
        \If{$\gamma \leq \tau$}:
            \State \textbf{Output:} $\mathsf{S}X + \gamma\textsc{C}_X\xi_1$ \ with $\mathsf{S} \sim \pN(0, \brI_{k\times n})$ and $\xi_1 \sim \pN(0, \brI_{k\times d})$. 
        \Else:
            \State Set $\widetilde{\lambda} = \max\left\{\lambda_{\min}(X^{T} X) - \eta\textsc{C}^{2}_X\left(\tau - \sz\right), 0\right\}$, where $\sz \sim \pN(0, 1)$. 
            \State Set $\widetilde{\eta} = \sqrt{\max\{\gamma - \widetilde{\lambda}, 0\}}$.
            \State \textbf{Output:} $\mathsf{S}X + \widetilde{\eta}\textsc{C}_X\xi_1$ \ with $\mathsf{S} \sim \pN(0, \brI_{k\times n})$ and $\xi_1 \sim \pN(0, \brI_{k\times d})$.
        \EndIf
        \end{algorithmic}
        \label{alg:gauss_mix_full}
        \end{algorithm}
    \end{minipage}
\end{figure*}

\begin{thm}
    \label{thm:Privacy_First}
    Let $\delta \in (0,1)$, $k\geq 1$, $\eta > 0$ and $\gamma > 5/2$, and set $\tau = \sqrt{2\log(\nicefrac{3}{\delta})}$. Then, the output of \algorithmref{alg:gauss_mix_full} is $(\widetilde{\eps}(\eta, \gamma, k, \delta),\delta)$-\DP, where  
    \begin{align}
        \label{eq:eps_tot_final}
        \widetilde{\eps}(\eta,\gamma, k, \delta) = \frac{\sqrt{2\log(\nicefrac{3.75}{\delta})}}{\eta} + \underset{1 < \alpha < \gamma}{\min}\left\{\varphi(\alpha; k,\gamma) + \frac{\log(\nicefrac{3}{\delta}) + (\alpha - 1)\log(1 - \nicefrac{1}{\alpha}) -\log(\alpha)}{\alpha - 1}\right\}. 
    \end{align}
\end{thm}

\begin{proof}[Proof sketch]
    The case of $\gamma < \tau$ is the previously discussed  $\overline{\lambda}_{\min} = 0$ setting. For the other case, we note that w.p. $\ge 1 - \nicefrac{\delta}{3}$ we have $\widetilde{\lambda} \ge \lambda_{\min}(X^{T} X)$, so the conditions of Lemma \ref{lem:Privacy_First} are met. Since $\overline{\lambda}_{\min}$ has sensitivity $\textsc{C}_{X}^2$, the privacy properties of the Gaussian mechanism apply to the release of $\widetilde{\lambda}$. Using simple composition, we combine it with the privacy of \ref{eq:mechanism_intro}, which completes the proof.
    
\end{proof} 
As we show in \appendixref{app:minimization_exist}, $\widetilde{\eps}(\eta,\gamma, k, \delta) \geq 0$ and further can be made arbitrarily small by increasing $\eta$ and $\gamma$. As mentioned previously, the parameter $\gamma$ should be thought of as the target minimal eigenvalue, $\tau$ serves as the accuracy bound on the estimation of $\overline{\lambda}_{\min}$, and $\eta$ controls the privacy loss of its private estimation. Setting $\eta = \gamma / \sqrt{k}$, and combining Corollary \ref{corol:tcdp_proof} with the DP implication of tCDP, we get that \algorithmref{alg:gauss_mix_full} is $(\eps, \delta)$-DP where $\eps = \frac{k}{2\gamma^{2}} + \frac{2\sqrt{2 k \log(\nicefrac{4}{\delta})}}{\gamma}$, which matches the privacy bound on \ref{eq:mechanism_intro} up to a constant. 

%%%%%%%%%%%%%%%%%%%%%%%%%%%%%%%%%%%%%%%%%%%%%%%%%%%%%%%%%%%%%%%%%%%%%%%%%%%%%%%%%%%%%%%%%%%%%%%

\section{Applications} 
\label{s:apps}

We now demonstrate applications of \ref{eq:mechanism_intro} for two private learning tasks that can be formalized as linear regression. In each case, we instantiate the mechanism on a concatenated dataset \((X_1, X_2)\) where \(X_{1} \in \reals^{n \times d_{1}}\) and \(X_{2} \in \reals^{n \times d_{2}}\).
The application of \ref{eq:mechanism_intro} to the concatenation \((X_{1}, X_{2})\) is given by
\begin{align}
    \pM(X_1, X_2) = \mathsf{S}\left(X_1, X_2\right) + \sigma\left(\xi_1, \xi_2\right) = \left(\mathsf{S}X_1 + \sigma\xi_1, \mathsf{S}X_2 + \sigma \xi_2\right)
\end{align}
where \(\mathsf{S} \sim \pN(0, \brI_{k \times n})\), \(\xi_1 \sim \pN(0, \brI_{k \times d_1})\), and \(\xi_2 \sim \pN(0, \brI_{k \times d_2})\).
If \(\pM(X_1, X_2)\) satisfies \RDP w.r.t. \((X_{1}, X_{2})\), then by post-processing the inner-product \(\pK\) defined as
\begin{align}
    \label{eq:grad_based_mechanism}
    \pK(\pM(X_1, X_2)) \coloneqq \left(\mathsf{S}X_1 + \sigma \xi_1\right)^{\top}\left(\mathsf{S}X_2 + \sigma \xi_2\right)
\end{align} 
also satisfies \RDP. The inner-product \(\pK\) will form the core component of our algorithms. 

%%%%%%%%%%%%%%%%%%%%%%%%%%%%%%%%%%%%%%%%%%%%%%%%%%%%%%%%%%%%%%%%%%%%%%%%%%%%%%%%%%%%%%%%%%%%%%
\subsection{Differentially Private Ordinary Least Squares}
\label{s:linear}

We begin with the problem of DP linear regression. Let the dataset \((X, Y)\) such that the design matrix \(X \in \reals^{n \times d}\) and the response vector \(Y \in \reals^n\).
Throughout, we make the following assumptions.
\begin{enumerate}[leftmargin=0.75in, itemsep=0pt]     
    \item [(\textbf{A}\textsubscript{1}) ] \emph{Bounded domain: } $\norm{x_i} \leq \textsc{C}_X$ and $|y_i| \leq \textsc{C}_Y$ for all $i \in [n]$\footnote{These domain bounds appear in many standard DP linear regression settings, such as \citet{sheffet2017differentially, wang_adassp}.}. 
    \item [(\textbf{A}\textsubscript{2}) ] \emph{Overspecified system: } $n \geq d$. 
\end{enumerate} 

Our goal is to estimate a linear predictor under $(\eps, \delta)$-DP, while preserving the privacy of each individual data pair $(x_i^{\top}, y_i)$. Our non-private baseline is the standard Ridge regression estimator \citep{tikhonov1963solution, hoerl1970ridge}: 
\begin{align}
    \label{eq:optimal_ridge}
    \theta^{*}(\nu) \coloneqq \underset{\theta}{\argmin}\ \left\{\|Y - X\theta\|_2^2 + \nu \|\theta\|_2^2\right\} = (X^{\top}X + \nu \brI_d)^{-1}X^{\top}Y,
\end{align}
where \(\nu \geq 0\) is a regularization parameter. The unregularized least-squares solution is denoted \(\theta^* \coloneqq \theta^*(0)\), and we define the empirical loss as \(L_{X,Y}(\theta) \coloneqq \|Y - X\theta\|_2^2\).
Our algorithm, summarized in \algorithmref{alg:linear_mix}, uses \algorithmref{alg:gauss_mix_full} for obtaining a privatized version of the pair $(X,Y)$ (interpreted as halves of one joint matrix) by setting a large enough $\gamma$ and then solving the resulting least-squares problem.  The existence of a $\gamma$ that satisfies the conditions in Line \ref{line:find_gamma} in \algorithmref{alg:linear_mix}  is ensured by the analysis presented in \appendixref{app:minimization_exist}. While the structure of our algorithm resembles earlier proposal by \citet{sheffet2017differentially}, our refined privacy analysis under \RenyiDP improves the overall privacy--utility trade-off. Moreover, our algorithm exploits $\lambda_{\min}$ in a modified way, ensuring that the variance of the additive noise component is always reduced by utilizing the private estimation of $\lambda_{\min}(X^{\top}X)$.

\begin{figure*}[t]
    \centering
    \begin{minipage}[t]{\textwidth}
        \begin{algorithm}[H]
        \caption{\texttt{LinearMixing}} 
        \begin{algorithmic}[1]
        \Require Dataset $(X,Y)$ satisfying assumptions (\textbf{A}\textsubscript{1}, \textbf{A}\textsubscript{2}), privacy parameters $(\eps, \delta)$, parameter $k$.
        
        \State\label{line:find_gamma}Find the smallest $\gamma > \nicefrac{5}{2}$ such that $\widetilde{\eps}(\eta, \gamma, k, \delta)$ (given in \eqref{eq:eps_tot_final}) is less than $\eps$\label{alg:find-gamma}, while setting $\eta = \nicefrac{\gamma}{\sqrt{k}}$. 
        
        \State Calculate $\left[\widetilde{X}, \widetilde{Y}\right] = \texttt{ModifiedGaussMix}\Big([X,Y], \sqrt{\textsc{C}_{X}^{2} + \textsc{C}_{Y}^2}, k, \gamma, \sqrt{2\log\left(\nicefrac{3}{\delta}\right)}, \eta\Big)$.   
        
        \State \textbf{Output:} $\theta_{\mathrm{Lin}} \coloneqq \big(\widetilde{X}^{\top}\widetilde{X}\big)^{-1}\widetilde{X}^{\top}\widetilde{Y}$. 
        \end{algorithmic}
        \label{alg:linear_mix}
        \end{algorithm}
    \end{minipage}
\end{figure*}

\begin{thm}
    \label{thm:linear_reg_privacy}
    For any $k \geq 1$ the output $\theta_{\mathrm{Lin}}$ is $(\eps,\delta)$-\DP. 
\end{thm}
\begin{proof}
    We first note that for any \(k \geq 1\), we are guaranteed to find a \(\gamma\) in Line~\ref{alg:find-gamma} such that the target \(\widetilde{\eps}(\eta, \gamma, k, \delta) \leq \eps\) where \(\eta = \nicefrac{\gamma}{\sqrt{k}}\) (see \appendixref{app:minimization_exist}).
    Since we apply \texttt{ModifiedGaussMix} (Line 2) with the appropriate domain bounds, the differential privacy guarantee follows from \theoremref{thm:Privacy_First}.
    Finally, since \(\theta_{\mathrm{Lin}}\) is a function of \((\widetilde{X}, \widetilde{Y})\), which is \((\eps,\delta)\)-DP, we have by post-processing that \(\theta_{\mathrm{Lin}}\) also satisfies \((\eps,\delta)\)-DP.
\end{proof}

\begin{figure}[t]
  \centering
  % Row 1
  \begin{subfigure}[t]{0.425\textwidth}
    \centering
    \includegraphics[width=\linewidth]{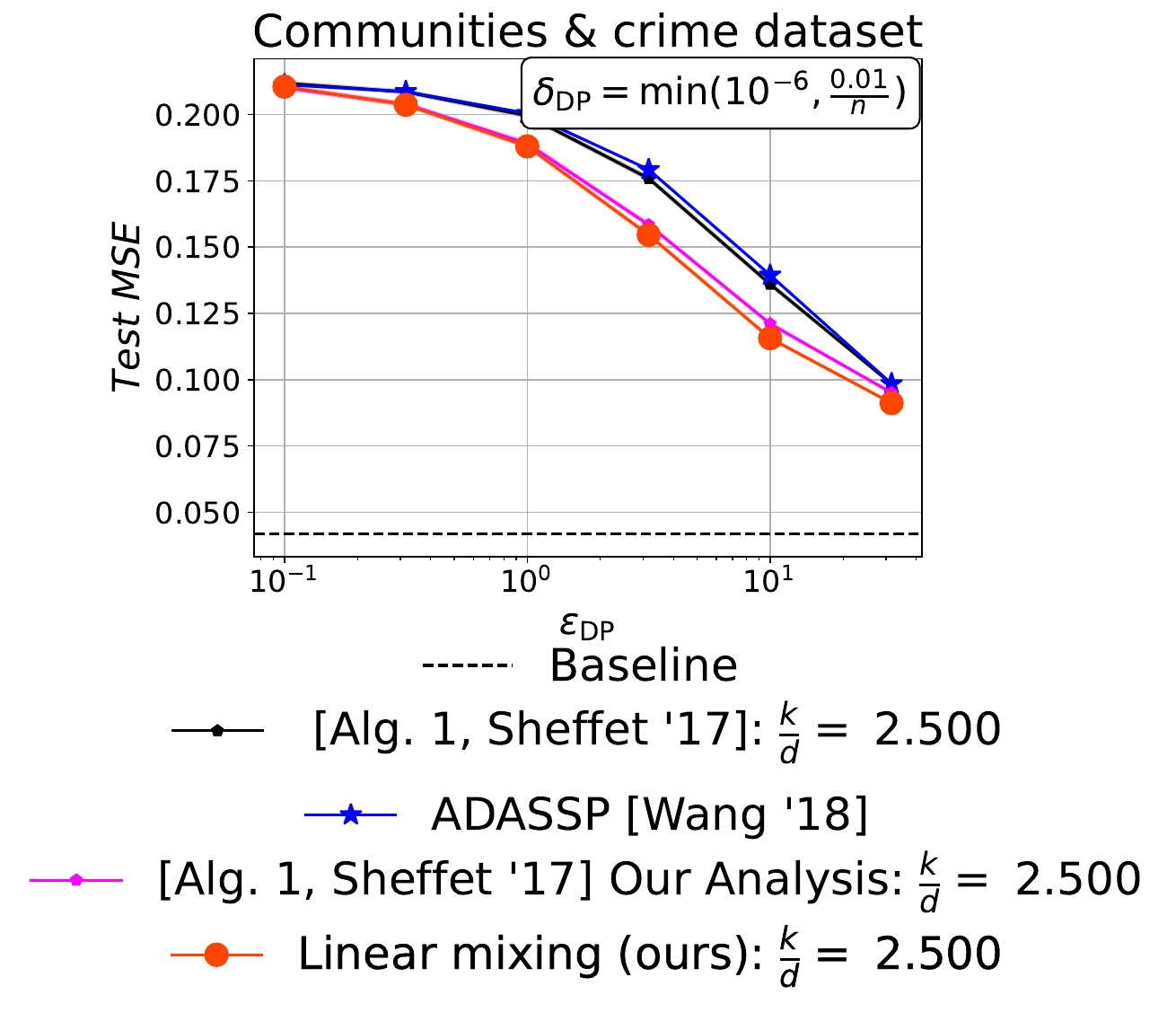}
    % \caption{Crime dataset}
  \end{subfigure}
  \hspace{0.02\textwidth}
  \begin{subfigure}[t]{0.425\textwidth}
    \centering
    \includegraphics[width=\linewidth]{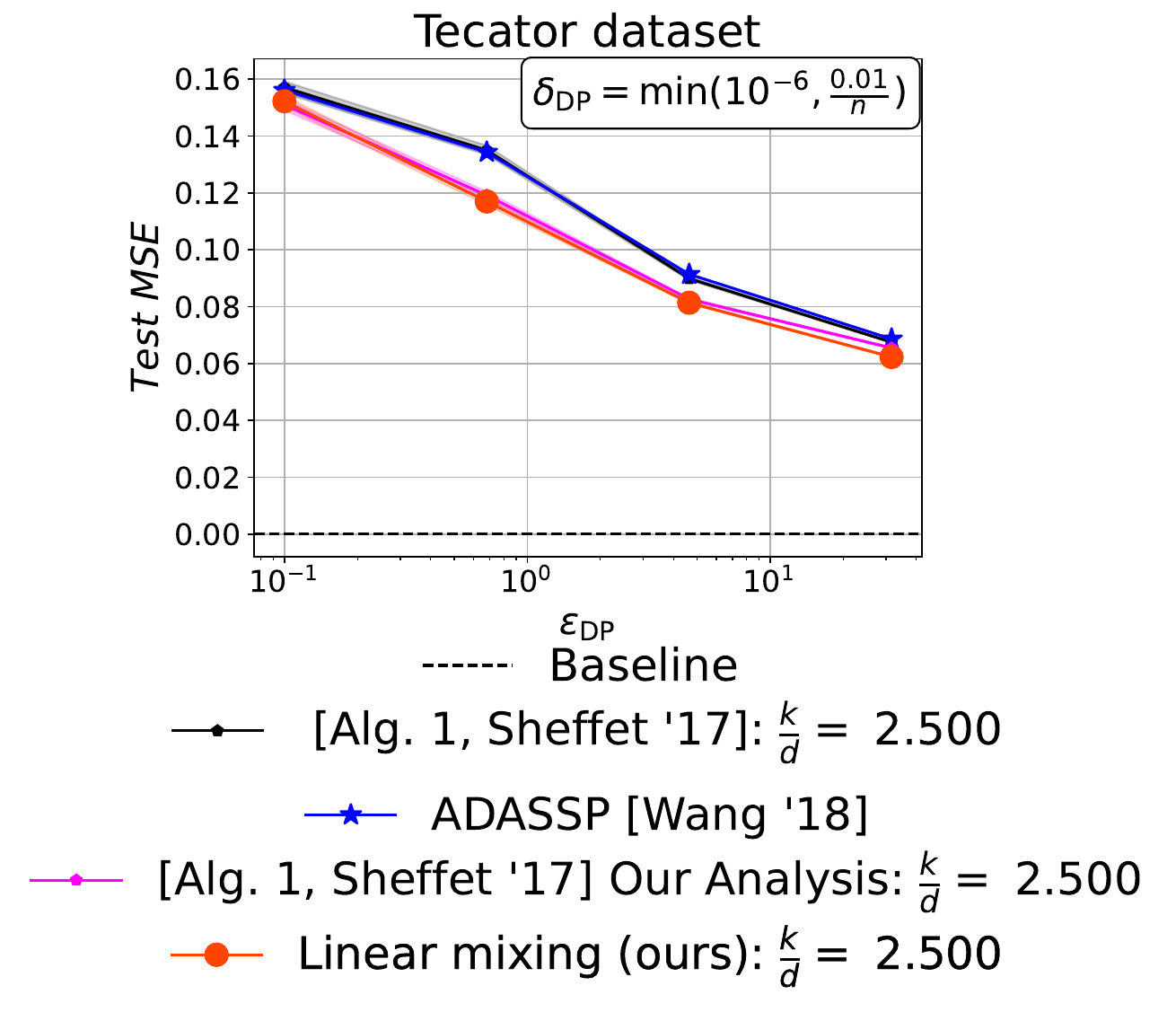}
  \end{subfigure}

  \vspace{0.5em} 
  % Row 2
  \begin{subfigure}[t]{0.425\textwidth}
    \centering
    \includegraphics[width=\linewidth]{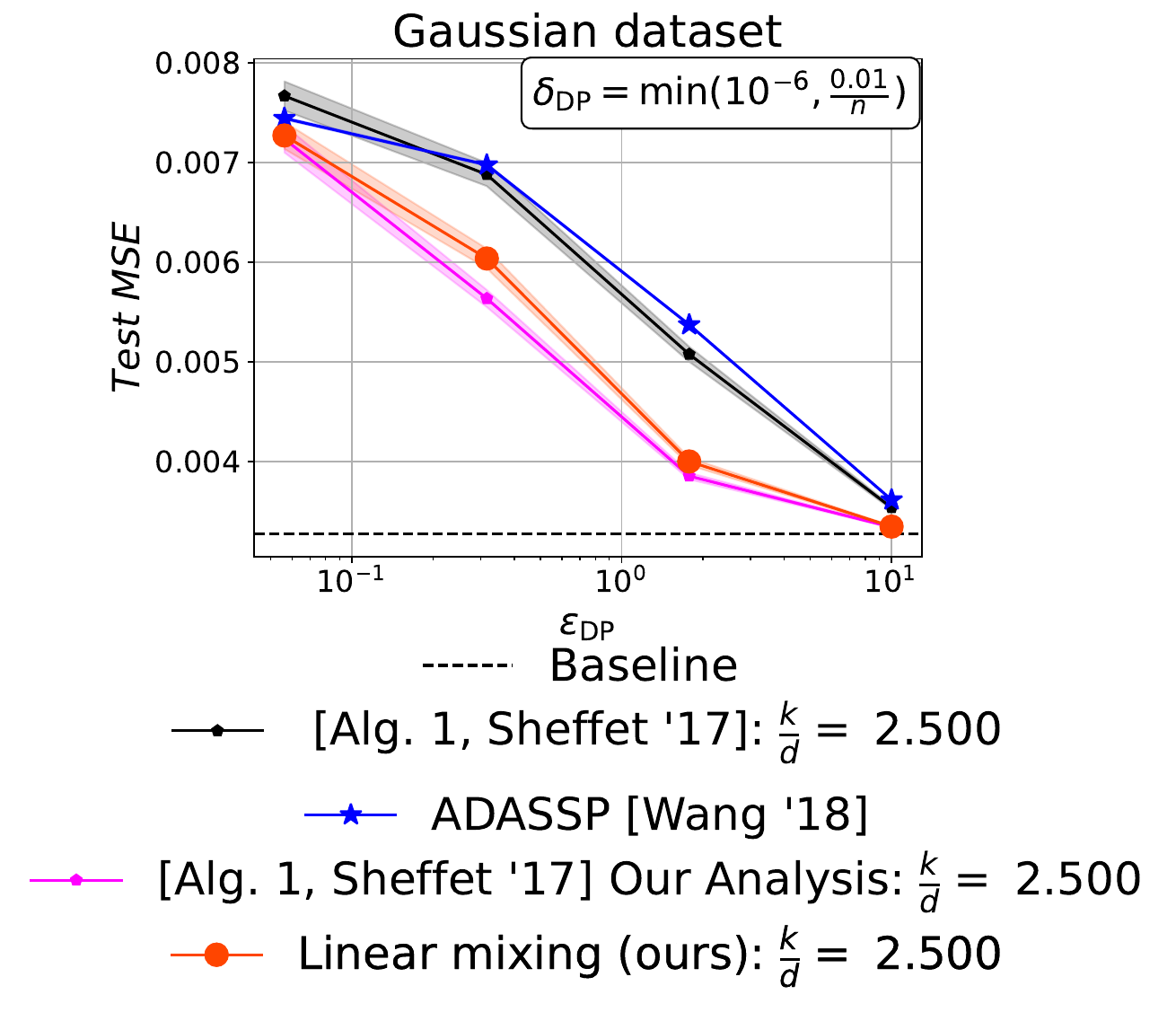}
  \end{subfigure}
  \hspace{0.02\textwidth}
  \begin{subfigure}[t]{0.425\textwidth}
    \centering
    \includegraphics[width=\linewidth]{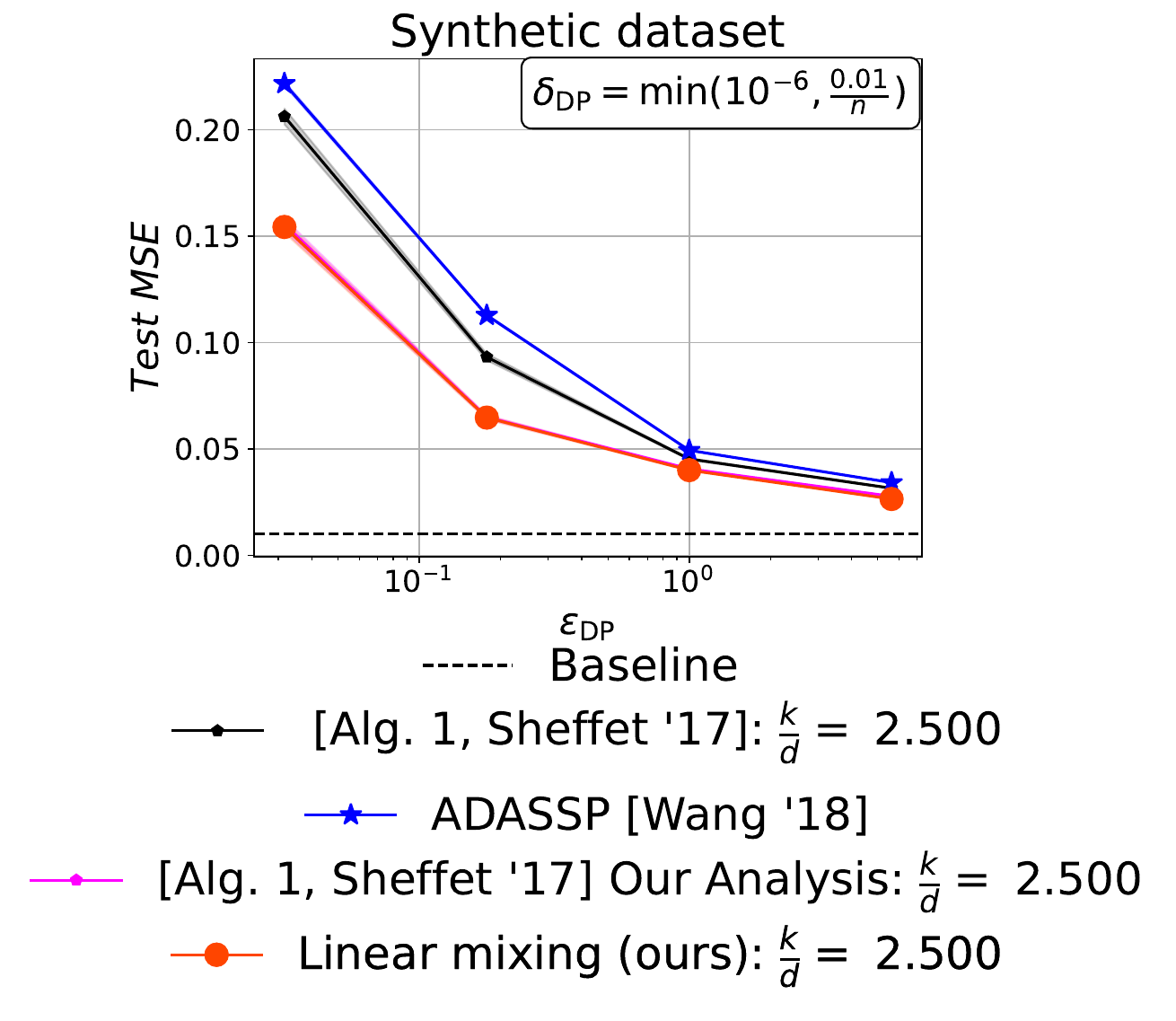}
  \end{subfigure}

  \caption{Linear mixing performance on four different regression datasets. The parameter \(k\) was chosen to ensure negligible approximation error in the sketched solution.}
  \label{fig:linear}
\end{figure}

Our next result provides performance guarantees for the solution returned by \algorithmref{alg:linear_mix}.

\begin{thm}
    \label{thm:linear_reg}
    There exist universal constants \( c_0, c_1, c_2 \) such that for any \( \chi \) satisfying \( k\chi^2 > c_0 d \), the following holds with probability at least \( 1 - c_1 \cdot \exp\left\{-c_2 k\chi^2\right\}\):  
    \begin{align}           
        L_{X,Y}(\theta_{\mathrm{Lin}}) - (1 + \chi)^2 \cdot L_{X,Y}(\theta^*) \leq O\!\left(\textstyle(1 + \chi)^2 \cdot \frac{\sqrt{k\log(\nicefrac{1}{\delta})} \cdot (\textsc{C}^{2}_X + \textsc{C}^{2}_Y)}{\eps} \cdot \left(1 + \norm{\theta^*}^2\right)\right). 
    \end{align} 
\end{thm}

We give the proof of the above theorem in Appendix~\ref{app:proof_linear}.
Notably, in the regime where \(\textsc{C}_Y \approx \textsc{C}_X, \norm{\theta^{*}} \gg 1 \) and \(d \gg 1\), for target probability greater than $1 - \varrho$ and with \(\varrho < c_1 \cdot \exp(-c_0 c_2 d)\), setting $\chi = (e_1d)^{-\nicefrac{1}{2}}$ and $k = (\nicefrac{e_1 d}{c_2})\cdot\log\left(\nicefrac{c_1}{\varrho}\right)$ for some constant $e_1 > 1$, yields excess empirical risk of $O\Big(\tfrac{\sqrt{d\log(\nicefrac{1}{\varrho})\cdot \log(\nicefrac{1}{\delta})}\cdot \textsc{C}^{2}_X}{\eps} \cdot \norm{\theta^{*}}^2\Big)$ matching the guarantees attained by AdaSSP~\citep{wang_adassp} for the case where $\lambda_{\min}(X^{\top}X) = 0$. 

%%%%%%%%%%%%%%%%%%%%%%%%%%%%%%%%%%%%%%%%%%%%%%%%%%%%%%%%%%%%%%%%%%%%%%%%%%%%%%%%%%%%%%%
We empirically demonstrate the behavior of \algorithmref{alg:linear_mix} by applying it to four datasets,  plotted in \figureref{fig:linear}.
Two of these datasets are real-world datasets: the Communities \& Crime dataset \citep{redmond2002data}, and the Tecator dataset \citep{Tecator_bib}.
The other two datasets are synthetic datasets: one where the covariates are Gaussian random vectors, and another where the covariates are transformations of Gaussian random vectors by a multi-layer perceptron.
We elaborate on the complete empirical setup in \appendixref{app:details} and present additional simulations for another four datasets in \appendixref{s:lin_reg_additional}.
We compare the performance of \algorithmref{alg:linear_mix} to three baseline methods: Algorithm 1 in \cite{sheffet2017differentially} and its variant guided by our analysis of \ref{eq:mechanism_intro} (\lemmaref{lem:Privacy_First}), and the AdaSSP algorithm \citep{wang_adassp}.
The AdaSSP algorithm is considered the leading baseline for DP linear regression under bounded domain assumptions \citep{liu2022differential, brown2024private}.
Recent works (for e.g., \citep{ferrando2024private}) have proposed empirical enhancements over AdaSSP at a significantly higher computational cost and without sound theoretical utility guarantees.
In comparison, \algorithmref{alg:linear_mix} takes just as long as AdaSSP while achieving better solutions for the same level of privacy.
This is seen in \figureref{fig:linear} where our method outperforms the baselines on the datasets we consider.
Interestingly, we see that the modification to the algorithm in \cite{sheffet2017differentially} based on our analysis in \lemmaref{lem:Privacy_First} for \ref{eq:mechanism_intro} causes an improvement in its performance, thereby highlighting the benefits of our analysis.

%%%%%%%%%%%%%%%%%%%%%%%%%%%%%%%%%%%%%%%%%%%%%%%%%%%%%%%%%%%%%%%%%%%%%%%%%%%%%%%%%%%%%%%
\subsection{Differentially Private Logistic Regression}
\label{s:logistic}

\begin{figure}[t]
  \centering
  \begin{subfigure}[b]{0.4\linewidth}
    \includegraphics[width=\linewidth]{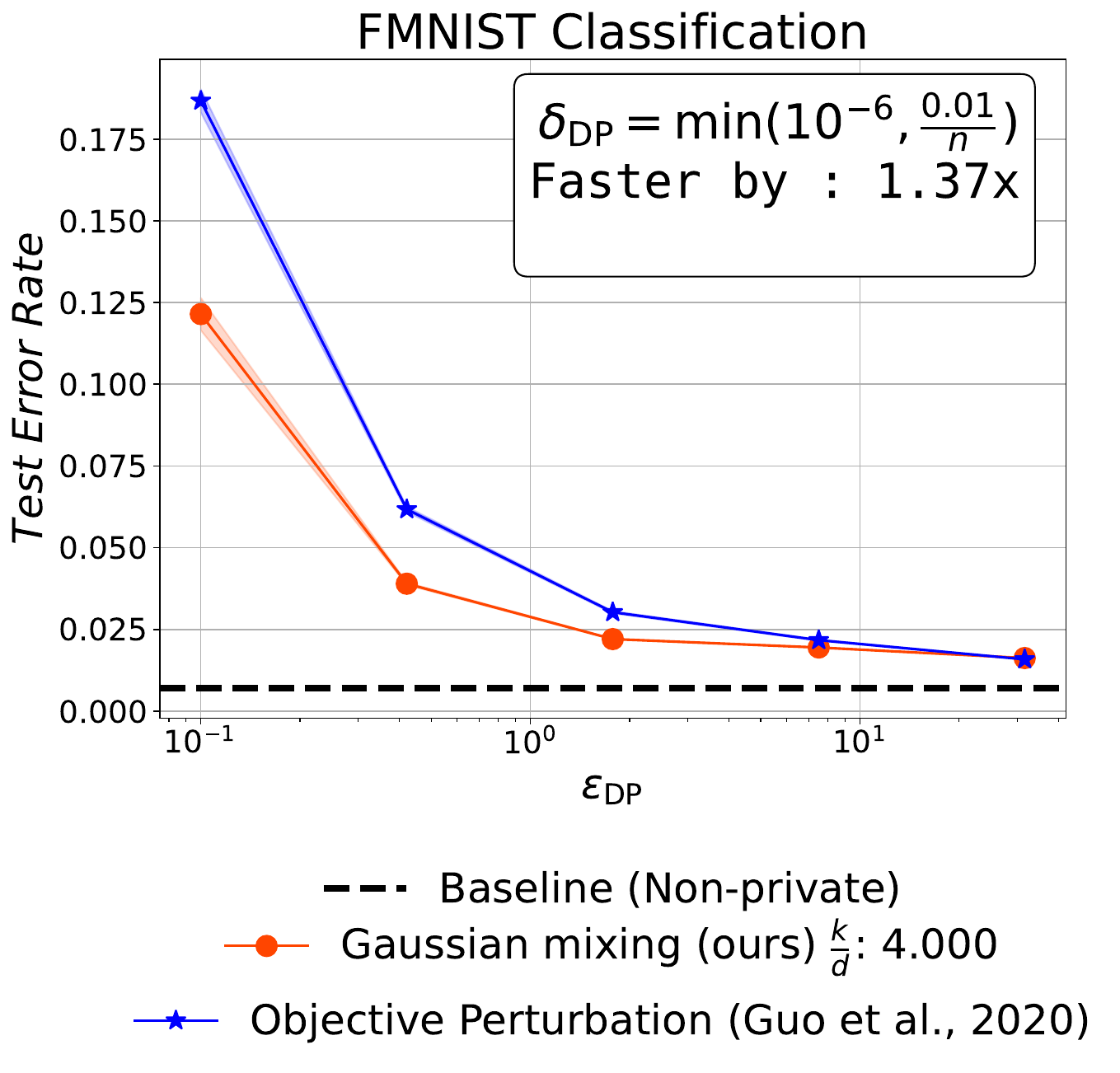}
  \end{subfigure}
  \hfill
  \begin{subfigure}[b]{0.4\linewidth}
    \includegraphics[width=\linewidth]{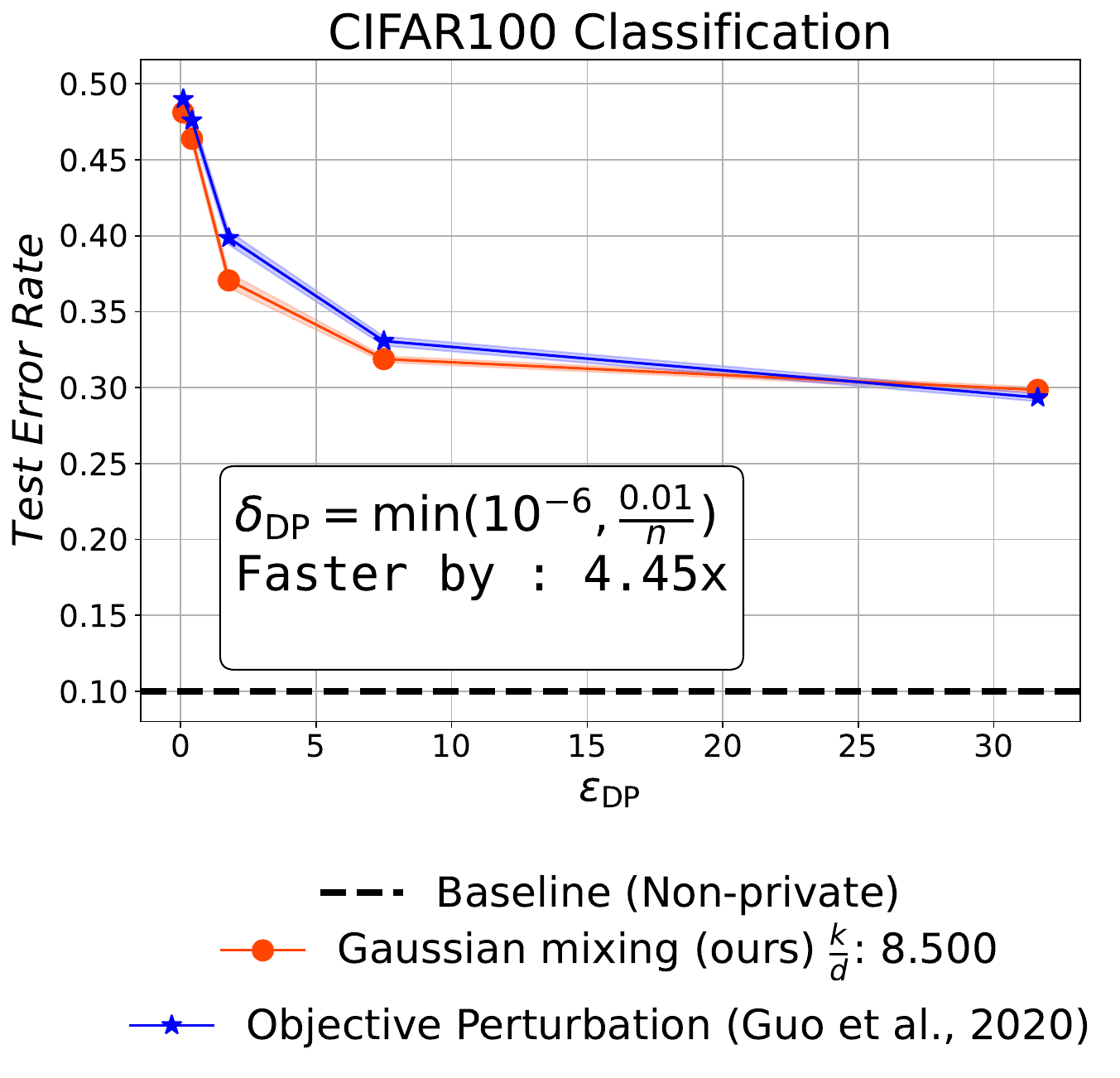}
  \end{subfigure}
  
  \caption{DP logistic regression using a privately trained CNN feature extractor on binary subsets of Fashion-MNIST and CIFAR100. The parameter $k$ was selected via grid search to maximize performance at the largest simulated $\varepsilon$.}
  \label{fig:logistic_test}
\end{figure}

Our second application consider the problem of DP logistic regression, with the goal of developing DP solutions to the ERM problem:
\begin{align}
    \theta^{*} = \underset{\theta}{\argmin} \ \sum^{n}_{i=1}-\log\left(1 + \exp\left\{-y_i \theta^{\top}x_i\right\}\right) := \underset{\theta}{\argmin} \ \sum^{n}_{i=1}\ell(s_i, y_i)
\end{align}
and where $y_i\in\{-1, +1\}$ for all $i\in[n]$. Following \citep{huggins2017pass} (see also \citep{ferrando2024private}), this problem can be solved by approximating
$\ell(s,y)$ with a second-order polynomial
$
    q(s) = b_0 + b_1s + b_2s^{2},$ i.e. $-\log\left(1+\exp\{-s\}\right) \approx q(s),  \text{for } s\in\mathcal{I},
$
where $\mathcal{I}$ satisfies \(y \theta^{\top} x \in  \mathcal{I}\) for all datapoints \((x, y)\), ensuring the surrogate is valid over the dataset.
Substituting $s_i=y_i\,\theta^\top x_i$ and discarding the constant
$b_0$, minimising the surrogate objective reduces to ordinary least
squares with a response vector $\widetilde{Y}\;=\;-\frac{b_1}{2b_2}\,Y,$ with $  Y=(y_1,\dots,y_n)^\top.$
Therefore, we can invoke \algorithmref{alg:linear_mix}, originally
designed for linear regression, to obtain a \((\eps,\delta)\)-DP
estimate for logistic regression. The utility of this solution follows via similar arguments to those of \theoremref{thm:linear_reg} and is presented in \appendixref{app:logistic}. However, in this case, the complexity of sketching the data and then solving the linear system (for example, using a QR decomposition) will be $O(nkd + kd^2)$. In many cases, this one-shot approach is more computationally efficient than classical approaches that use iterative optimization techniques. 

To demonstrate this point, we have tested our approach in a similar setting to that presented by \citet{GuoCertifiedDataRemoval}, where we (i) train a CNN with DP stochastic gredient descent (DP-SGD) implemented using Opacus \citep{yousefpour2021opacus} and (ii) use the pre-trained private embeddings for DP fine-tuning of a logistic head. The CNN architecture and training hyperparameters mirror those in \citep{GuoCertifiedDataRemoval} (see \appendixref{app:details} for full details). We train the head either with objective perturbation as in \citep{GuoCertifiedDataRemoval} or with our approach. As demonstrated in \figureref{fig:logistic_test}, our approach offers computational improvement over the objective perturbation baseline, and also accuracy improvement.

%%%%%%%%%%%%%%%%%%%%%%%%%%%%%%%%%%%%%%%%%%%%%%%%%%%%%
%%%%%%%%%%%%%%%%%%%%%%%%%%%%%%%%%%%%%%%%%%%%%%%%%%%%%%%%%%%%%%%%%%%%%%%%%%%%%%%%%%%%%%%%%%%%%%%%%%%%%%%%%%%%%%%%%%%%%%%%

\section{Discussion and Conclusion} 
In this work, we revisited the Gaussian mixing mechanism originally introduced by \citet{blocki2012johnson}, and later studied by \citet{sheffet2017differentially, sheffet2019old} and \citet{showkatbakhsh2018privacy}. We derived its \RDP curve, which yields tighter bounds on the relationship between the noise parameter \(\sigma\) and the the minimal eigenvalue $\lambda_{\min}(X^{\top}X)$ and \((\eps,\delta)\), thereby strengthening the privacy analysis of this mechanism. We further demonstrated the practical usefulness of this improved analysis by applying \ref{eq:mechanism_intro} to two distinct machine learning tasks and providing: an algorithm for DP (i)  linear regression, (ii) logistic regression. 

The analysis we provide in this work is tight and improves over the analysis presented by \citet{sheffet2019old}. Consequently, it also offers performance improvement in other settings currently invoking the results of \citet{sheffet2019old}, such as \citet{bartan2023distributed}, and also in similar settings that relies on \ref{eq:mechanism_intro} and analyze it via alternative privacy metrics, such as \citet{Coded_Yona, CodedFederated_MIDP}. 

Moreover, a key technical property that underpins the usefulness of \ref{eq:mechanism_intro} is its compatibility with formulations in which terms involving the random projection matrix $\mathsf{S}$ cancel in expectation, as described intuitively in \eqref{eq:intuition_utility}. Identifying additional formulations of machine learning problems that naturally admit this structure and whose implementation with \ref{eq:mechanism_intro} facilitates an improved privacy–utility trade-off, along with lower computational complexity, is an interesting direction for future work. This opens the possibility of applying \ref{eq:mechanism_intro} beyond the currently studied linear regression settings. 

Finally, this work focuses on the case where $\mathsf{S}$ is Gaussian. Extending the framework to other classes of projections, such as those studied in the broader sketching literature \citep{woodruff2014sketching, pilanci2015randomized, pilanci_hessiansketch}, could enable more efficient implementations of \eqref{eq:grad_based_mechanism} and facilitate scalable private learning. 

\section*{Acknowledgements} 
AS thanks Gautam Kamath for helpful discussions. The work of MS and KL was supported in part by ERC grant 101125913, Simons Foundation Collaboration 733792, Israel Science Foundation (ISF) grant 2861/20, Apple, and a grant from the Israeli Council of Higher Education. Views and opinions expressed are however those of the author(s) only and do not necessarily reflect those of the European Union or the European Research Council Executive Agency. Neither the European Union nor the granting authority can be held responsible for them.

%%%%%%%%%%%%%%%%%%%%%%%%%%%%%%%%%%%%%%%%%%%%%%%%%%%%%%%%%%%%%%%%%%%%%%%%%%%%%%%

%%%%%%%%%%%%%%%%%%%%%%%%%%%%%%%%%%%%%%%%%%%%%%%%%%%%%%%%%%%%%%%%%%%%%%%%%%%%%%%
\newpage
\bibliographystyle{plainnat} 
% \bibliography{Config/Bib_Shrtened}

\appendix

%%%%%%%%%%%%%%%%%%%%%%%%%%%%%%%%%%%%%%%%%%%%%%%%%%%%%%%%%%%%%%%%%%%%%%%%%%%%%%%
\newpage 
  
\section{Notation}
\label{app:notation}
Given a matrix $A\in \reals^{m\times n}$ we denote its elements by $A_{ij}$ and its column-stack representation by 
\begin{align}
    \text{vec}(A) \coloneqq (A_{11},A_{21},\ldots, A_{m1},A_{12},A_{22},\ldots,A_{mn})^{\top}.    
\end{align}
Random variables are denoted using sans-serif fonts (e.g., \(\mathsf{X}, \mathsf{y}\)), while their realizations are represented by regular italics (e.g., \(X, y\)). The $L_2$ norm of a vector $A \coloneqq (a_1,\ldots, a_d)$ is given by $\left(\sum^{d}_{i=1}a^2_i\right)^{\frac{1}{2}}$ and is denoted by $\norm{A}$. We denote the minimal and the maximal eigenvalues of a matrix $A$ by $\lambda_{\min}(A)$ and $\lambda_{\max}(A)$. We denote a PSD matrix $A$ by $A \succeq 0$ and a PD matrix by $A\succ0$. We usually denote our dataset $\left\{x_i\right\}^{n}_{i=1}$ where each $x_i \in \reals^{d}$ in the matrix form $X = (x_1,\ldots, x_n)^{\top}$. Then, we denote the $j$'th entry of $x_i$ by $x_i(j)$. The determinant of a matrix $A$ is denoted by $\deter{A}$. The set of integer numbers from $1$ to $n$ is denoted by $[n]$. The all zeros column vector of size $d$ is denoted by $\vec{0}_d \coloneqq (0,\ldots, 0)^{\top}$. We denote by $\pN(0, \brI_{k_1\times k_2})$ a $k_1\times k_2$ matrix comprised of \iid \ Gaussian elements with zero mean and unit variance. We denote by $\pN_{\text{sym}}(0, \brI_{d})$ a $d\times d$ symmetric matrix whose elements on the upper triangular matrix are \iid \ and distributed accoring to $\pN(0,1)$. The Kronecker product between matrices $A$ and $B$ is defined via 
\begin{align}
    A \otimes B \coloneqq \begin{pmatrix}
    A_{11}B & \ldots & A_{1m}B\\
     & \ddots & \\
     A_{n1}B & \ldots & A_{nm}B
    \end{pmatrix}.
\end{align}
The $k\times k$ identity matrix is denoted by $\brI_{k}$. 
%%%%%%%%%%%%%%%%%%%%%%%%%%%%%%%%%%%%%%%%%%%%%%%%%%%%%%%%%%%%%%%%%%%%%%%%%%%
\section{Proof of \lemmaref{lem:Privacy_First}}
\label{app:privacy}
\label{app:}

The proof relies on the next lemma, which establishes the $\alpha$-\Renyi divergence between two multivariate Gaussian distributions. 

\begin{lem}[\citet{gil2013renyi}] 
    Let $\sx_1 \sim \pN(\mu_1,\Sigma_1)$ and $\sx_2 \sim \pN(\mu_2, \Sigma_2)$. Then,
    \begin{align}
        \label{eq:Gaussian_Div}
        \KLDA{\sx_1}{\sx_2} &= \frac{\alpha}{2}(\mu_1 - \mu_2)^{\top} \left(\Sigma_1 + \alpha(\Sigma_2 - \Sigma_1)\right)^{-1}(\mu_1 - \mu_2) \\
        &\qquad \hspace{1in} - \frac{1}{2(\alpha - 1)} \log \left(\frac{\deter{\Sigma_1 + \alpha(\Sigma_2 - \Sigma_1)}}{\left(\deter{\Sigma_1}\right)^{1-\alpha}\left(\deter{\Sigma_2}\right)^{\alpha}}\right)\ \ \ 
    \end{align}
    for all $\alpha$ such that $\alpha \Sigma^{-1}_1 + (1-\alpha)\Sigma^{-1}_2 \succ 0$. 
\end{lem}

\begin{proof}[Proof of \lemmaref{lem:Privacy_First}] 
Instead of analyzing \ref{eq:mechanism_intro}, we will analyze the transformed mechanism 
\begin{align}
    \widetilde{\pM}(X) \coloneqq  (\pM(X))^{\top} = X^{\top}\mathsf{S}^{\top} + \sigma \xi^{\top}, 
\end{align}
which, in terms of privacy, is equivalent to $\pM(X)$. We note that $\widetilde{\pM}(X)$ is a matrix of Gaussian random variables, where its columns are \iid \ and each column has a covariance of 
\begin{align}
    \E{(X^{\top}\mathsf{s}_{i} + \sigma \xi_i)(X^{\top}\mathsf{s}_{i} + \sigma \xi_i)^{\top}} =  X^{\top}X + \sigma^2 \brI_d = \sum^{n}_{i=1}x_{i}x^{\top}_{i} + \sigma^2 \brI_d
\end{align}
where we have denoted $\mathsf{S}^{\top} = (\mathsf{s}_1,\ldots, \mathsf{s}_k)$ and $\xi^{\top} = (\xi_1,\ldots, \xi_k)$. 
Thus, we first note that
\begin{align}
    \text{vec}(\widetilde{\pM}(X)) \sim \pN(0, \brI_{k}\otimes ( X^{\top}X + \sigma^2 \brI_{d})).
\end{align}

Let $X'$ be our neighbor dataset that is different from $X$ in a single row. Throughout we assume that $X'$ is equivalent to $X$ except for one row which is zeroed out (see \sectionref{s:prelims}). We then show that the proof also covers the inverse case where one row of $X$ is zeroed out.

Without loss of generality, assume that the differing row is the first row of $X$. Thus, we first note that
\begin{align}
    \brI_k\otimes( X^{\top}X + \sigma^2 \brI_d) -  \brI_k\otimes( X'^{\top}X' + \sigma^2 \brI_d)= \brI_{k}\otimes ( x_1x^{\top}_1). 
\end{align}

Let $\Sigma_1 \coloneqq  \brI_k\otimes( X^{\top}X + \sigma^2 \brI_d)$ and $\Sigma_2 \coloneqq  \brI_k\otimes( X'^{\top}X' + \sigma^2 \brI_d)$. Now,  
\begin{align} 
    \Sigma_1 + \alpha(\Sigma_2 - \Sigma_1) = \brI_k \otimes \left(-\alpha  x_1x^{\top}_1 +  X^{\top}X + \sigma^2 \brI_d\right)
\end{align}
and since $\E{\text{vec}(\widetilde{\pM}(X))} = \E{\text{vec}(\widetilde{\pM}(X'))} = 0$ by using the algebraic identity $\deter{\brI_k \otimes A} = (\deter{A})^{k}$ and by using \eqref{eq:Gaussian_Div} we get 
\begin{align}
    \label{eq:Renyi_Bound_mid}
    \KLDA{\pM(X)}{\pM(X')} &= -\frac{k}{2(\alpha - 1)}\log\left(\frac{\deter{-\alpha  x_1x^{\top}_1 +  X^{\top}X + \sigma^2 \brI_d}}{(\deter{ X^{\top}X + \sigma^2 \brI_d})^{1-\alpha}(\deter{ X^{\top}X + \sigma^2 \brI_d -  x_1x_1^{\top}})^{\alpha}}\right)\ .
\end{align}

For an invertible matrix $A$, we have \citep[Section.~3.4]{Brookes2020Matrix} 
\begin{align}
    \deter{A + uv^{\top}} = \deter{A}(1 + v^{\top}A^{-1}u).
\end{align}

Since the matrix $ X^{\top}X + \sigma^2 \brI_d$ is invertible whenever $\sigma^2 > 0$, this further tells us that the denominator of \eqref{eq:Renyi_Bound_mid} can be simplified to 
\begin{align}
    & \left(\deter{ X^{\top}X + \sigma^2 \brI_d}\right)^{1-\alpha}\left(\deter{ X^{\top}X + \sigma^2 \brI_d -  x_1x_1^{\top}}\right)^{\alpha} \\
    &\qquad\qquad\qquad = \left(\deter{ X^{\top}X + \sigma^2 \brI_d}\right)^{1-\alpha}\left(\deter{ X^{\top}X + \sigma^2 \brI_d}\right)^{\alpha}\left(1 -  x^{\top}_1(X^{\top}X + \sigma^2 \brI_d)^{-1}x_1\right)^{\alpha}\\
    &\qquad\qquad\qquad = \deter{ X^{\top}X + \sigma^2 \brI_d} \left(1 -  x^{\top}_1(X^{\top}X + \sigma^2 \brI_d)^{-1}x_1\right)^{\alpha}.
\end{align}

Thus, we further have 
\begin{align}
    &\frac{\deter{-\alpha x_1x^{\top}_1 +  X^{\top}X + \sigma^2 \brI_d}}{(\deter{ X^{\top}X + \sigma^2 \brI_d})^{1-\alpha}(\deter{ X^{\top}X + \sigma^2 \brI_d -  x_1x_1^{\top}})^{\alpha}} \\
    &\qquad = \left(\frac{\deter{-\alpha x_1x^{\top}_1 +  X^{\top}X + \sigma^2 \brI_d}}{\deter{ X^{\top}X + \sigma^2 \brI_d}}\right)\cdot(1 -  x^{\top}_1( X^{\top}X + \sigma^2 \brI_d)^{-1}x_1)^{-\alpha}.
\end{align}
Similarly, we can apply the same determinant identity to $\deter{-\alpha x_1x^{\top}_1 +  X^{\top}X + \sigma^2 \brI_d}$ and get 
\begin{align}
    \frac{\deter{-\alpha x_1x^{\top}_1 +  X^{\top}X + \sigma^2 \brI_d}}{\deter{ X^{\top}X + \sigma^2 \brI_d}} = 1 - \alpha x^{\top}_1 ( X^{\top}X + \sigma^2 \brI_d)^{-1}x_1.
\end{align}

Next, we note that this yields the next simplified form for $\KLDA{\pM(X)}{\pM(X')}$:
\begin{align}
    \label{eq:KLDA_Base_Mechanism_Equality}
    &\KLDA{\pM(X)}{\pM(X')} \\
    &\qquad= -\frac{k}{2(\alpha-1)}\log\left((1 - \alpha  x^{\top}_1 ( X^{\top}X + \sigma^2 \brI_d)^{-1}x_1)(1 -  x^{\top}_1 ( X^{\top}X + \sigma^2 \brI_d)^{-1}x_1)^{-\alpha}\right)\\
    \label{eq:KLDA_Base_Mechanism_Equality_simplified}
    &\qquad= \frac{k}{2(\alpha-1)}\log\left(\frac{(1 -  x^{\top}_1 ( X^{\top}X + \sigma^2 \brI_d)^{-1}x_1)^{\alpha}}{1 - \alpha x^{\top}_1 ( X^{\top}X + \sigma^2 \brI_d)^{-1}x_1}\right).
\end{align}
We note that the function $\frac{(1-t)^{\alpha}}{1 - \alpha t}$ is a monotonically non-decreasing function of $t$ in the range $0\leq t < \frac{1}{\alpha}$ for $\alpha > 1$. To see this, note that 
\begin{align}
    \frac{\partial}{\partial t} \log\left(\frac{(1-t)^{\alpha}}{1 - \alpha t}\right) &= \frac{\partial}{\partial t} \left\{\alpha \log(1 - t) - \log(1 - \alpha t) \right\} = -\frac{\alpha}{1 - t} + \frac{\alpha}{1 - \alpha t} = \frac{\alpha (\alpha - 1)t}{(1-t)(1-\alpha t)} 
\end{align}
which is positive in the range $0 \leq t < \frac{1}{\alpha}$ (recall that $\alpha > 1$). Thus, to further simplify \eqref{eq:KLDA_Base_Mechanism_Equality_simplified}, we will try to find an upper bound on $x^{\top}_1(X^{\top}X + \sigma^2 \brI_d)^{-1}x_1$. To that end, note that for a general symmetric positive-definite matrix $A$ we have 
\begin{align}
    \label{eq:Op_Norm_Bound}
    x^{\top}A^{-1}x \leq \frac{\norm{x}^2}{\lambda_{\min}(A)}
\end{align}
where equality is achieved whenever $x$ is the eigenvector of $A$ that correspond to $\lambda_{\min}(A)$. 
Then, using this relation with regard to $ X^{\top}X + \sigma^2 \brI_d\succ 0$ and using the identity
\begin{align}
    \label{eq:LowerBound_Op_Norm} 
    \lambda_{\min}( X^{\top}X + \sigma^2 \brI_d) =  \lambda_{\min}(X^{\top}X) + \sigma^2
\end{align}
yields
\begin{align}
    x^{\top}_i( X^{\top}X + \sigma^2 \brI_d)^{-1}x_i \leq \frac{\norm{x_i}^2}{\lambda_{\min}(X^{\top}X) + \sigma^2}, \ \ \ \ \text{for all} \ \ i=1,\ldots, n. 
\end{align}
Since we know that $\underset{i \in [n]}{\max} \ \norm{x_i}^2 \leq \textsc{C}^{2}_X$ we further have
\begin{align}
     x^{\top}_i( X^{\top}X + \sigma^2 \brI_d)^{-1}x_i \leq \frac{\textsc{C}^{2}_X}{\lambda_{\min}(X^{\top}X) + \sigma^2}.
\end{align}

This further leads to the next final upper bound on the \Renyi divergence: 
\begin{subequations}
\begin{align}
    \label{eq:FinalBound_Matrix_Renyi}
    &\KLDA{\pM(X)}{\pM(X')} \\
    &\qquad = \frac{k}{2(\alpha-1)}\log\left(\frac{(1 -  x^{\top}_1 ( X^{\top}X + \sigma^2 \brI_d)^{-1}x_1)^{\alpha}}{1 - \alpha  x^{\top}_1 ( X^{\top}X + \sigma^2 \brI_d)^{-1}x_1}\right) \\
    \label{eq:FinalBound_Matrix_Renyi_mid_step}
    &\qquad\leq \frac{k}{2(\alpha - 1)}\log\left(\frac{(1 - \frac{\norm{x_1}^2}{\lambda_{\min}(X^{\top}X)+\sigma^2})^{\alpha}}{1 - \frac{\alpha \norm{x_1}^2}{ \lambda_{\min}(X^{\top}X) + \sigma^2}}\right)\\
    &\qquad= \frac{k\alpha}{2(\alpha -1)}\log\left(1 - \frac{\norm{x_1}^2}{\lambda_{\min}(X^{\top}X) + \sigma^2}\right) - \frac{k}{2(\alpha - 1)}\log\left(1 - \frac{\alpha \norm{x_1}^2}{\lambda_{\min}(X^{\top}X) + \sigma^2}\right), 
\end{align}
\end{subequations}
where \eqref{eq:FinalBound_Matrix_Renyi_mid_step} requires that $\alpha < \frac{\lambda_{\min}(X^{\top}X) + \sigma^2}{\norm{x_1}^2}$. Then, since a similar analysis holds when we replace $x_1$ with a general point $x_i$, the worst case divergence between $X$ and an $X'_i$ that is changed by zeroing out one entry $x_i$ is 
\begin{subequations}
\begin{align}
    &\underset{i\in [n]}{\sup} \ \KLDA{\pM(X)}{\pM\left(X'_i\right)} \\
    &\quad \leq \underset{i\in[n]}{\sup} \ \left\{\frac{k\alpha}{2(\alpha -1)}\log\left(1 - \frac{\norm{x_i}^2}{\lambda_{\min}(X^{\top}X) + \sigma^2}\right) - \frac{k}{2(\alpha - 1)}\log\left(1 - \frac{\alpha  \norm{x_i}^2}{\lambda_{\min}(X^{\top}X) + \sigma^2}\right)\right\}\\
    \label{eq:FinalBound_Matrix_Renyi_closedform}
    &\quad\leq \frac{k\alpha}{2(\alpha -1)}\log\left(1 - \frac{\textsc{C}^{2}_X}{\lambda_{\min}(X^{\top}X) + \sigma^2}\right) - \frac{k}{2(\alpha - 1)}\log\left(1 - \frac{\alpha \textsc{C}^{2}_X}{\lambda_{\min}(X^{\top}X) + \sigma^2}\right)\\
    \label{eq:Renyi_final_bound_final}
    &\quad\leq \frac{k\alpha}{2(\alpha -1)}\log\left(1 - \frac{\textsc{C}^{2}_X}{\overline{\lambda}_{\min} + \sigma^2}\right) - \frac{k}{2(\alpha - 1)}\log\left(1 - \frac{\alpha \textsc{C}^{2}_X}{\overline{\lambda}_{\min} + \sigma^2}\right), 
\end{align}
\end{subequations}
where \eqref{eq:Renyi_final_bound_final} is again by the monotonicity of $\frac{(1-t)^{\alpha}}{1 - \alpha t}$ and since $\lambda_{\min}(X^{\top}X) \geq \overline{\lambda}_{\min}$, where $\alpha \leq \underset{i}{\min} \ \left\{\frac{\overline{\lambda}_{\min} + \sigma^2}{\norm{x_i}^2}\right\} = \frac{\overline{\lambda}_{\min} + \sigma^2}{\textsc{C}^{2}_X}$ and the bound holds whenever $\overline{\lambda}_{\min} + \sigma^2 > \textsc{C}^{2}_X$. Finally, note that since $\alpha - 1 > 0$ for all $\alpha > 1$ and since $\alpha \log\left(1 - \frac{\textsc{C}^{2}_X}{\overline{\lambda}_{\min} + \sigma^2}\right) - \log\left(1 - \frac{\alpha \textsc{C}^{2}_X}{\overline{\lambda}_{\min} + \sigma^2}\right) \geq 0$ for all $\alpha > 1$ (this follows since the function is $0$ for $\alpha = 1$ and since its derivative is positive) this upper bound is non-negative and is a valid upper bound on this divergence. 

For the case where one row of $X$ is zeroed out, we note that we have $X'^{\top}X' + \sigma^2\brI_d = X^{\top}X + \sigma^2\brI_d + x_ix^{\top}_i$. Then, \eqref{eq:KLDA_Base_Mechanism_Equality_simplified} is replaced with
\begin{align}
    \KLDA{\pM(X)}{\pM(X')} = \frac{k}{2(\alpha-1)}\log\left(\frac{(1 +  x^{\top}_1 ( X^{\top}X + \sigma^2 \brI_d)^{-1}x_1)^{\alpha}}{1 + \alpha x^{\top}_1 ( X^{\top}X + \sigma^2 \brI_d)^{-1}x_1}\right).
\end{align}
Now, we define the function $f(t;\alpha) = \log\left(\frac{(1-t)^{\alpha}}{1-\alpha t}\right) - \log\left(\frac{(1+t)^{\alpha}}{1+\alpha t}\right)$. Then, note that $f(0;\alpha) = 0$ and further since $\alpha > 1$ and $\alpha t < 1$ then   
\begin{align}
    \frac{\partial}{\partial t}f(t;\alpha) = 2\alpha \left(\frac{1}{1 - (\alpha t)^2} - \frac{1}{1 - t^2}\right)\geq 0
\end{align}
and thus $f(t;\alpha) \geq 0$ for all $t<\frac{1}{\alpha}$, and we get that the maximum between the two divergences is always given by the case where $X'$ contains a zero row. Thus, by finding the $\sigma^2$ that makes \eqref{eq:Renyi_final_bound_final} equal to $\eps$ we guarantee that our mechanism is $(\alpha,\eps)$-\RenyiDP.

It remains to validate that the condition $\alpha \Sigma^{-1}_1 + (1-\alpha)\Sigma^{-1}_2 \succ 0$ holds. However, since throughout we have $\Sigma_2 = \Sigma_1 - x_ix^{\top}_i$ with $\Sigma_1 = X^{\top}X + \sigma^2\brI_d$, by using the formulas for the inverse of a rank-1 update we get 
\begin{align}
    \alpha \Sigma^{-1}_1 + (1-\alpha)\Sigma^{-1}_2 &= \Sigma^{-1/2}_1\left(\brI_d + (1-\alpha)\cdot \frac{\Sigma^{-1/2}_1x_ix^{\top}_i\Sigma^{-1/2}_1}{1 - x^{\top}_i\Sigma^{-1}_1x_i}\right)\Sigma^{-1/2}_1.
\end{align}
We note that since $\Sigma_1 \succ 0$, for this term to be positive definite it suffices for the middle matrix to be positive definite. However, since this matrix is a rank-1 update of $\brI_d$, its eigenvalues are $1$'s and an additional eigenvalue that is given by 
\begin{align}
    1 + (1-\alpha)\cdot \frac{\norm{\Sigma^{-1/2}_{1}x_i}^2}{1 - x^{\top}_i\Sigma^{-1}_1x_i} = 1 + (1-\alpha)\cdot \frac{x^{\top}_i\Sigma^{-1}_1x_i}{1 - x^{\top}_i\Sigma^{-1}_1x_i}.
\end{align}
We note that this term is positive whenever $\alpha \leq \frac{1}{x^{\top}_i\Sigma^{-1}_1x_i}$. However, since $x^{\top}_i\Sigma^{-1}_1x_i \leq \frac{\norm{x_i}^2}{\overline{\lambda}_{\min} + \sigma^2}$ this inequality is satisfied by the restrictions we have on the domain of $\alpha$.   

\end{proof}

%%%%%%%%%%%%%%%%%%%%%%%%%%%%%%%%%%%%%%%%%%%%%%%%%%%%%%%%%%%
\section{Proof of \corollaryref{corol:tcdp_proof}}
\label{app:proof-tcdp}

\begin{proof}
We start by defining the difference function
\[
\Delta(k,\alpha,\gamma)
=\frac{k\alpha}{2\gamma^{2}}
-\frac{k\alpha}{2(\alpha-1)}\log\left(1-\frac{1}{\gamma}\right)
+\frac{k}{2(\alpha-1)}\log\left(1-\frac{\alpha}{\gamma}\right).
\]
Our goal is to find when $\Delta(k, \alpha, \gamma) \geq 0$ for $1 < \alpha < \gamma$. 
We note multiplying by the positive factor \(2\gamma^{2}(\alpha-1)\) and cancelling the term \(k>0\) gives the equivalent condition
\[
G(\alpha)
\coloneqq\alpha(\alpha-1)
-\alpha \gamma^2\log\left(1-\frac{1}{\gamma}\right)
+\gamma^2\log\left(1-\frac{\alpha}{\gamma}\right)
\ge 0,\qquad 1<\alpha<w,
\]
where \(w\) will be specified shortly. On $\alpha = 1$ we further get $G(1) = 0$. Moreover,
\[
G'(\alpha)
=2\alpha-1
-\gamma^2\log\left(1-\frac{1}{\gamma}\right)
-\frac{\gamma^2}{\gamma-\alpha},
\]
and multiplying by \((\gamma - \alpha)>0\) (recall that $\alpha < \gamma$) shows \(G'(\alpha)\) has the same sign as the quadratic
\begin{align}
    H(\alpha) &\coloneqq(\gamma-\alpha)G'(\alpha) \\
    &= (2\alpha - 1)(\gamma - \alpha) - \gamma^2(\gamma-\alpha)\log\left(1 - \frac{1}{\gamma}\right) - \gamma^2\\
    &=-2\alpha^{2} +\left(1+2\gamma+\gamma^2\log\left(1-\frac{1}{\gamma}\right)\right)\alpha -\gamma\left(1 + \gamma + \gamma^2\log\left(1-\frac{1}{\gamma}\right)\right).
\end{align}

We define the discriminant to be 
\begin{align}
    \Delta_H = \left(1+2\gamma+\gamma^2\log\left(1-\frac{1}{\gamma}\right)\right)^{2} - 8\gamma\left(1+\gamma+\gamma^2\log\left(1-\frac{1}{\gamma}\right)\right)
\end{align}
which is non-negative. Thus, \(H\) has two real roots
\begin{align}
    \alpha_{\max/\min} = \frac{\gamma^2\log\left(1 - \nicefrac{1}{\gamma}\right) + 2\gamma + 1\pm\sqrt{\Delta_H}}{4},
\end{align}
and \(H(\alpha)\geq 0\) for \(\alpha\in[\alpha_{\min},\alpha_{\max}]\) since the coefficient of the quadratic term $H(\alpha)$ is negative. However, note that $\alpha_{\min} < 1/2$ for all $\gamma > 1$ and moreover $\alpha_{\max} > 1$ for $\gamma > 5/2$ and $\alpha_{\max} < \gamma$ for all $\gamma > 2.5$. Thus, since the derivative is positive and since $G(1) = 0$, setting \(w\coloneqq\alpha_{\max}\) yield \(G'(\alpha)\ge 0\) for every \(1<\alpha<w\) and thus the inequality \(G(\alpha)\ge 0\) holds throughout that interval, whenever $\gamma > 5/2$. The proof is completed since $\alpha_{\max} > \frac{2\gamma}{5}$ for all $\gamma > 1$. 
\end{proof}

%%%%%%%%%%%%%%%%%%%%%%%%%%%%%%%%%%%%%%%%%%%%%%%%%%%%%%%%%%%%%%%%%%%%%%%%%%%
\section{Proof of \theoremref{thm:Privacy_First}}
\label{app:privacy_full}
We recall that the sensitivity of the minimum eigenvalue $\lambda_{\min}(X^{\top}X)$ is $\textsc{C}^{2}_X$ (see, for example \citet{sheffet2017differentially, wang_adassp}). Then, by using the standard formula of the Gaussian mechanism \citep[Appendix ~A]{Dwork_AlgFoundationd_DP} we get that $\widetilde{\lambda}$ is $(\sqrt{2\log(3.75/\delta)}/\eta,\delta/3)$ release of $\lambda_{\min}(X^{\top}X)$. Using \lemmaref{lem:Privacy_First} and \propositionref{prop:Renyi_classical_translate}, we note that whenever $\lambda_{\min}(X^{\top}X) + \widetilde{\eta}^2 \geq \gamma$ the release of the output in both cases satisfies $(\widetilde{\eps}, \delta/3)$-DP where
\begin{align}
    \widetilde{\eps} = \underset{1 < \alpha < \gamma}{\min}\ \left\{\varphi(\alpha; k,\gamma) + \frac{\log\left(\nicefrac{3}{\delta}\right) + (\alpha - 1)\log\left(1 - \nicefrac{1}{\alpha}\right) - \log(\alpha)}{\alpha - 1}\right\}.
\end{align}
The first case (when $\gamma \leq \tau$) trivially satisfies this. However, for the second case (whenever $\gamma > \tau$), this is satisfied only if $\widetilde{\eta}^2 + \lambda_{\min}(X^{\top}X) \geq \gamma$, which by using the inequality  
\begin{align}
    \widetilde{\eta}^2 + \lambda_{\min}(X^{\top}X) = \gamma - \lambda_{\min}(X^{\top}X) + \eta\textsc{C}^{2}_X\tau - \eta\textsc{C}^{2}_X\sz + \lambda_{\min}(X^{\top}X) = \gamma + \eta\textsc{C}^{2}_X\tau - \eta\textsc{C}^{2}_X\sz
\end{align}
corresponds to having $\sz \geq \tau$ (we note that the case $\widetilde{\lambda} = 0$ immediately satisfies $\lambda_{\min}(X^{\top}X) + \widetilde{\eta}^2 \geq \gamma$ since then we have $\widetilde{\eta}^2 = \gamma$). 
Thus, 
\begin{align}
    P\left(\widetilde{\eta}^2 + \lambda_{\min}(X^{\top}X) \leq \gamma\right)
    = P\left(\sz \geq \tau\right) \leq \exp\left\{-\frac{\tau^2}{2}\right\} = \frac{\delta}{3}.
\end{align}
Then, using simple composition \cite[Chapter.~3]{Dwork_AlgFoundationd_DP} and substituting $\tau = \sqrt{2\log(\nicefrac{3}{\delta})}$ yields the desired result. 

%%%%%%%%%%%%%%%%%%%%%%%%%%%%%%%%%%%%%%%%%%%%%%%%%%%%%%%%%%%%%%%%%%%%%%%%%%%%%%%%%%%%%%%%%%%%%%
\section{Minimizing $\widetilde{\eps}(\eta,\gamma, k,\delta)$ }
\label{app:minimization_exist}

We now show that $\widetilde{\eps}(\eta,\gamma,k,\delta) \geq 0$ and further that one can make $\widetilde{\eps}(\eta,\gamma,k,\delta)$ as small as any desirable $\eps$ by increasing $\eta$ and $\gamma$. We first note that $\varphi(\alpha; k,\gamma)$ is an upper bound on $\KLDA{\pM(X)}{\pM(X')}$. Thus, following the validity of the conversion from RDP to DP of \citep{canonne2020discrete}, the second term in $\widetilde{\eps}(\eta,\gamma,k,\delta)$ provides an upper bound on the privacy parameter $\eps$, and thus is non-negative. Since the first term in $\widetilde{\eps}(\eta,\gamma, k,\delta)$ is non-negative we get that the entire expression is non-negative.

To prove that $\widetilde{\eps}(\eta,\gamma,k,\delta)$ can be made arbitrarily small, note that by \corollaryref{corol:tcdp_proof} we know that $\varphi(\alpha; k,\gamma) \leq \frac{k\alpha}{2\gamma^2}$ for $1 < \alpha \leq \frac{2}{5}\gamma$ and provided that $\gamma > \frac{5}{2}$. However, we note that the minimum in $\widetilde{\eps}(\eta,\gamma,k,\delta)$ is upper bounded by 
\begin{align}  
    \label{eq:minimum_gamma}
    \underset{1<\alpha<2\gamma/5}{\min} \ \left\{\frac{k\alpha}{2\gamma^2}     + \frac{\log\left(\nicefrac{3}{\delta}\right)}{\alpha - 1}\right\} \leq \frac{k}{5\gamma} + \frac{\log(\nicefrac{3}{\delta})}{\frac{2}{5}\gamma - 1}
\end{align}
which is derived by substituting $\alpha = 2\gamma/5$. Thus, this minimum is monotonically decreasing in $\gamma$ and can be made arbitrarily small by increasing $\gamma$. The result then follows since the first term in \eqref{eq:eps_tot_final} is monotonically decreasing in $\eta$, and holds further in the case where $\eta = \frac{\gamma}{\sqrt{k}}$ by picking a sufficiently large $\gamma$. 

%%%%%%%%%%%%%%%%%%%%%%%%%%%%%%%%%%%%%%%%%%%%%%%%%%%%%%%%%%%%%%%%%%%%%%%%%%%%%%%%%%%%%%%%%%%%%%
\section{Proof of \theoremref{thm:linear_reg}}
\label{app:proof_linear}

\begin{proof}    
   We first establish the performance of a method that adds noise with a general level $\sigma$, namely, solutions of the form 
   \begin{align}
       \theta_{\mathrm{Lin}} \coloneqq ((\mathsf{S}X + \sigma \xi_1)^{\top}(\mathsf{S}X + \sigma \xi_1))^{-1}(\mathsf{S}X + \sigma \xi_1)^{\top}(\mathsf{S}Y + \sigma \xi_2)
   \end{align}
    with $\mathsf{S}\sim \pN(0, \brI_{k\times n}), \xi_1 \sim \pN(0, \brI_{k\times d})$ and $\xi_2 \sim \pN(0, \brI_{k})$ and where $Y \in \reals^{n\times 1}$ and $X\in \reals^{n\times d}$. 
    
   Then, we can rewrite $\theta_{\mathrm{Lin}}$ in the next form 

    \begin{align}
        \label{eq:joint_sketch_interpret}
        \theta_{\mathrm{Lin}} = \underset{\theta}{\argmin} \ \norm{(\mathsf{S}, \xi_1, \xi_2)\left(\begin{pmatrix}
            Y \\ \vec{0}_d \\ \sigma
        \end{pmatrix} - \begin{pmatrix}
            X \\ \sigma\brI_d \\ \vec{0}^{\top}_d
        \end{pmatrix}\theta\right)}^2. 
    \end{align}
    Now, since $\text{rank}((X^{\top}, \sigma\brI_d, \vec{0}_d)^{\top}) = d$ and since $\sigma^2 \geq 0$, by \citep[Corrolary.~2]{pilanci2015randomized}, whenever $k > \frac{c_0d}{\chi^2}$ w.p. at least $1 - c_1 \cdot \mathrm{exp}\left\{-c_2 k\chi^2\right\}$ we have 
    \begin{align}
        \label{eq:Pilanci_Sketching}
        L_{X,Y}(\theta_{\mathrm{Lin}}) + \sigma^2\norm{\theta_{\mathrm{Lin}}}^2 + \sigma^2 &\leq (1 + \chi)^{2}\left(\norm{Y - X\theta^{*}\left(\sigma^2\right)}^2 + \sigma^2\norm{\theta^{*}\left(\sigma^2\right)}^2 + \sigma^2\right).
    \end{align}
    We note that this further implies that 
    \begin{align}
        L_{X,Y}(\theta_{\mathrm{Lin}}) \leq (1+\chi)^2\left(\norm{Y - X\theta^{*}(\sigma^2)}^2 + \sigma^2\left(1 + \norm{\theta^{*}}^2\right)\right).
    \end{align}
    Thus, we can write 
    \begin{align}
        L_{X,Y}(\theta_{\mathrm{Lin}}) - (1 + \chi)^2L_{X,Y}(\theta^{*}) &\leq (1+\chi)^2\left(\norm{Y - X\theta^{*}(\sigma^2)}^2 - \norm{Y - X\theta^{*}}^2 + \sigma^2(1+\norm{\theta^{*}}^2)\right)\\
        &= O\left((1+\chi)^2\sigma^2 \left(1 + \norm{\theta^{*}}^2\right)\right) 
    \end{align}
    where the last equality is by \citet[Appendix.~B.2]{wang_adassp}. Now, we note that in both of the cases of the algorithm the magnitude of the added noise is at most $\gamma(\textsc{C}_{X}^{2} + \textsc{C}^2_Y)$, where $\gamma$ is determined by the calculation done in step~\ref{alg:find-gamma}. Thus, since the bound is monotonically increasing in $\sigma^2$ we can further use the upper bound  
    \begin{align}
        \label{eq:empirical_risk_bound_gamma}
        L_{X,Y}(\theta_{\mathrm{Lin}}) - (1 + \chi)^2L_{X,Y}(\theta^{*}) &= O\left((1+\chi)^2\gamma (\textsc{C}^{2}_X + \textsc{C}^{2}_Y)\left(1 + \norm{\theta^{*}}^2\right)\right).
    \end{align}
    
    We further note that 
    \begin{align}
        \label{}
        \eps(\sigma,\gamma,k,\delta) &= \frac{\sqrt{2\log(\nicefrac{3.75}{\delta})}}{\eta} + \underset{1 < \alpha < \gamma}{\min} \ \left\{\varphi(\alpha; k,\gamma) + \frac{\log(\nicefrac{3}{\delta}) + (\alpha - 1)\log(1 - \nicefrac{1}{\alpha}) - \log(\alpha)}{\alpha - 1}\right\}\\
        &\leq \frac{\sqrt{2k\log(\nicefrac{3.75}{\delta})}}{\gamma} + \underset{1 < \alpha < 2\gamma/5}{\min}\left\{\frac{k\alpha}{2\gamma^2} + \frac{\log(\nicefrac{3}{\delta})}{\alpha - 1}\right\}\\
        &= \frac{3\sqrt{2k\log(\nicefrac{3.75}{\delta})}}{\gamma}.
    \end{align}
    Thus, equating this upper bound to $\eps$ suggests further that $\gamma = O\left(\frac{\sqrt{k\log(\nicefrac{1}{\delta})}}{\eps}\right)$ and using this bound in \eqref{eq:empirical_risk_bound_gamma} leads to
    \begin{align}
        \label{eq:additive_bound_one_before_last}
        L_{X,Y}(\theta_{\mathrm{Lin}}) - (1 + \chi)^2L_{X,Y}(\theta^{*}) &= O\left(\left(1 + \chi\right)^2\cdot \frac{\sqrt{k\log(\nicefrac{1}{\delta})}(\textsc{C}^{2}_X + \textsc{C}^{2}_Y)}{\eps}\left(1 + \norm{\theta^{*}}^2\right)\right). 
    \end{align}
    The proof is finished since this holds for any $\chi$ under the constraints in the theorem. 
\end{proof}

We note that substituting $\chi = \sqrt{\frac{1}{e_1d}}$ and $k = \frac{e_1 d}{c_2}\log\left(\frac{c_1}{\varrho}\right)$ into the previous argument yields  
\begin{align}
     L_{X,Y}(\theta_{\mathrm{Lin}}) &- \left(1 + \sqrt{\frac{1}{e_1d}}\right)^2L_{X,Y}(\theta^{*}) \\
    &= O\left(\left(1 + \sqrt{\frac{1}{e_1d}}\right)^2\cdot \frac{\sqrt{d\log\left(\nicefrac{1}{\varrho}\right)\log(\nicefrac{1}{\delta})}(\textsc{C}^{2}_X + \textsc{C}^{2}_Y)}{\eps}\left(1 + \norm{\theta^{*}}^2\right)\right)
\end{align}
which holds with probability at least $1 - \varrho$ provided that $\varrho < c_1\cdot\mathrm{exp}\left\{-c_0c_2d\right\}$.  

%%%%%%%%%%%%%%%%%%%%%%%%%%%%%%%%%%%%%%%%%%%%%%%%%%%%%%%%%%%%%%%%%%%%%%%%%%%%%%%%%%%%%%%%%%%%%%
\section{Utility Guarantees for Logistic Regression}
\label{app:logistic}
We now demonstrate utility guarantees on our method for DP logistic regression, presented in \sectionref{s:logistic}. Those derived similarly to \theoremref{thm:linear_reg}, by considering both sources of errors: the error of approximating the objective with a polynomial and the empirical error of the linear regression solution. Throughout the proof, we denote by $\widehat{\theta}$ the private solution obtained by scaling the output of \algorithmref{alg:linear_mix} by $-\frac{b_1}{2b_2}$. We also let $\widetilde{\theta}^*$ denote the minimizer of the approximated loss, given explicitly by $-\frac{b_1}{2b_2}(X^{\top}X)^{-1}X^{\top}Y$.
 We further denote the empirical logistic loss via 
\begin{align}
    L_{X,Y}(\theta) \coloneqq -\frac{1}{n}\sum^{n}_{i=1}\log\left(1 + \exp\left\{-y_i x^{\top}_i\theta\right\}\right)
\end{align}
and the approximated empirical logistic loss via 
\begin{align}
    \widetilde{L}_{X,Y}(\theta) &\coloneqq b_0 + b_1 \theta^{\top} \left(\frac{1}{n}X^{\top}Y\right) + b_2 \theta^{\top}\left(\frac{1}{n}X^{\top}X\right)\theta \\
    &= b_0 - \frac{b^2_1}{4nb_2}\norm{Y}^2 + \frac{b_2}{n}\norm{-\frac{b_1}{2b_2}Y - X\theta}^2 \\
    &= b_0 - \frac{b^2_1}{4nb_2}\norm{Y}^2 + \frac{b_2}{n} F\left(X, -\frac{b_1}{2b_2}Y, \theta\right)
\end{align}
where $F(X,Y,\theta) \coloneqq \norm{Y - X\theta}^2$. 
We note that \theoremref{thm:linear_reg_privacy} guarantees that our logistic regression solution obtained by minimizing this surrogate is private. We now present the utility guarantees on this solution. 

\begin{corol}
    \label{corol:logistic}
    Assume that $\norm{x_i}^2_2 \leq \textsc{C}^{2}_X$, $\abs{y_i}^2 \leq \textsc{C}^{2}_Y$ and $\abs{y_i x^{\top}_i \widetilde{\theta}^{*}} \leq Q$ and $\abs{y_i x^{\top}_i \widehat{\theta}} \leq Q$ for all $i\in[n]$ and for some finite $Q > 0$. Let $(b_0, b_1, b_2)$ chosen such that 
    \begin{align}
        \label{eq:assumption_upperbound_q}
        \underset{s\in [-Q, Q]}{\max} \ \abs{-\log(1 + e^{-s}) - (b_0 + b_1s + b_2 s^2)} \leq q.
    \end{align}
    Then, there exist universal constants \( c_0, c_1, c_2 \) such that for any \( \chi \) satisfying \( k\chi^2 > c_0 d \) \ the following holds with probability at least \( 1 - c_1 \cdot \exp\left\{-c_2 k\chi^2\right\} \):
        \begin{align}
            &L_{X,Y}(\theta^{*})-(1+\chi)^2L_{X,Y}(\widehat{\theta}) - (1+(1+\chi)^2)q+(1-(1+\chi)^2)\left(b_0-\frac{b^2_1}{4b_2}\right)\\ 
            &\qquad\qquad\qquad\qquad =O\left(\textstyle(1+\chi)^2 \frac{\sqrt{k\log(\nicefrac{1}{\delta})}\textsc{C}^{2}_X}{n\eps} \|\widetilde{\theta}^*\|^2\right).
        \end{align}
\end{corol}

\begin{proof}
    We first note that 
    \begin{subequations}
    \begin{align}
        \label{eq:upper_bound_logistic_step1}
        L_{X,Y}(\theta^{*}) - (1+\chi)^2L_{X,Y}\left(\widehat{\theta}\right)&\leq L_{X,Y}(\widetilde{\theta}^{*}) - (1+\chi)^2L_{X,Y}\left(\widehat{\theta}\right)\\
        &= L_{X,Y}(\widetilde{\theta}^{*}) - \widetilde{L}_{X,Y}(\widetilde{\theta}^{*}) + \widetilde{L}_{X,Y}(\widetilde{\theta}^{*}) - (1+\chi)^2\widetilde{L}_{X,Y}\left(\widehat{\theta}\right)\\
        &\quad + (1+\chi)^2\widetilde{L}_{X,Y}\left(\widehat{\theta}\right) - (1+\chi)^2L_{X,Y}\left(\widehat{\theta}\right)\\
        \label{eq:upper_bound_logistic_step2}
        &\leq (1 + (1+\chi)^2)q + \widetilde{L}_{X,Y}(\widetilde{\theta}^{*}) - (1+\chi)^2\widetilde{L}_{X,Y}\left(\widehat{\theta}\right)\ \ \ \ \ \\
        \label{eq:upper_bound_logistic_step3}
        &= (1 + (1+\chi)^2)q + (1 - (1+\chi)^2)\left(b_0 - \frac{b^2_1}{4nb_2}\norm{Y}^2\right)\ \ \ \\
        &\qquad +\frac{b_2}{n}\left(F\left(X,-\frac{b_1}{2b_2}Y,\widetilde{\theta}^{*}\right) - (1+\chi)^2F\left(X,-\frac{b_1}{2b_2}Y,\widehat{\theta}\right)\right)
    \end{align}
    \end{subequations}
    where \eqref{eq:upper_bound_logistic_step1} is by the optimality of $\theta^{*}$, \eqref{eq:upper_bound_logistic_step2} is by \eqref{eq:assumption_upperbound_q} and \eqref{eq:upper_bound_logistic_step3} is by the definition of $\widetilde{L}_{X,Y}(\theta)$ and by the assumptions $\abs{y_i x^{\top}_i \widetilde{\theta}^{*}} \leq Q$ and $\abs{y_i x^{\top}_i \widehat{\theta}} \leq Q$. Then, the final result follows by using \theoremref{thm:linear_reg} and since in this case $\abs{y_i} = 1$, thus $\norm{Y} = n$ and $\textsc{C}_Y = 1$.   
\end{proof}

When we take $\chi \ll 1$, the bound acquires an extra $2q$ term in the excess empirical risk relative to the bound obtained in the linear regression case, introduced by the polynomial approximation.

%%%%%%%%%%%%%%%%%%%%%%%%%%%%%%%%%%%%%%%%%%%%%%%%%%%%%%%%%%%%%%%%%%%%%%%%%%%%%%%%%%%%%%%%%%%%%%
\section{Algorithms: Linear Regression}
\label{app:algorithms}  

\subsection{AdaSSP}
\label{app:AdaSSP}  

\begin{algorithm}[H]
\caption{AdaSSP \citep{wang_adassp}} 
\begin{algorithmic}[1]

\Require Dataset $(X,Y)$; Privacy parameters $\eps, \delta$; Bounds: $\underset{i\in [n]}{\max} \ \norm{x_i}^2 \leq \textsc{C}^{2}_X,  \underset{i\in [n]}{\max} \ \abs{y_i}^2 \leq \textsc{C}^{2}_Y$. 

\State Calculate the minimum eigenvalue $\lambda_{\min}(X^\top X)$.

\State Privately release $\widetilde{\lambda}_{\min} = \max\left\{ \lambda_{\min} + \frac{\sqrt{\log(\nicefrac{6}{\delta})}\textsc{C}^{2}_X}{\nicefrac{\eps}{3}}\sz - \frac{\log(\nicefrac{6}{\delta})}{\nicefrac{\eps}{3}} \textsc{C}^{2}_X, 0 \right\}
$ where  $\sz \sim \pN(0,1)$. 

\State Set $\lambda = \max\left\{0, \frac{\sqrt{d \log(\nicefrac{6}{\delta}) \log(2d^2/\rho)} \textsc{C}^{2}_X}{\nicefrac{\eps}{3}} - \widetilde{\lambda}_{\min} \right\}$.

\State Privately release $\widetilde{X^\top X} = X^\top X + \frac{\sqrt{\log(\nicefrac{6}{\delta})}\textsc{C}^{2}_X}{\nicefrac{\eps}{3}} \xi_1$ for $\xi_1 \sim \pN_{\mathrm{sym}}(0, \brI_{d})$.

\State Privately release $\widetilde{X^\top y} = X^\top y + \frac{\sqrt{\log(\nicefrac{6}{\delta})} \textsc{C}_X\textsc{C}_Y}{\nicefrac{\eps}{3}} \xi_2$ for $\xi_2 \sim \pN(0, \brI_d)$. 

\State \Return $\widetilde{\theta} \gets (\widetilde{X^\top X} + \lambda \brI_d)^{-1} \widetilde{X^\top y}$
\end{algorithmic}
\end{algorithm}

%%%%%%%%%%%%%%%%%%%%%%%%%%%%%%%%%%%%%%%%%%%%%%%%%%%%%%%%%%%%%%%%%%%%%%%%%%%%%%%%
\subsection{Algorithm 1 from \citet{sheffet2017differentially}} 
\label{app:sheffet} 

\begin{algorithm}[H]
\caption{Sheffet's Algorithm~\citep[Algorithm~1]{sheffet2017differentially}} 
\begin{algorithmic}[1]

\Require Dataset $(X,Y)$; Privacy parameters $\varepsilon, \delta$; Bounds: $\underset{i \in [n]}{\max}\ \norm{x_i}^2 \leq \textsc{C}^{2}_X$, $\underset{i \in [n]}{\max}\ |y_i|^2 \leq \textsc{C}^{2}_Y$; Hyperparameter $k$.

\State Compute $\lambda_{\min} := \lambda_{\min}((X, Y)^\top (X, Y))$.

\State Set $\gamma \gets \frac{4(\textsc{C}^{2}_X + \textsc{C}^{2}_Y)}{\varepsilon} \left( \sqrt{2k\log\left(\frac{8}{\delta}\right)} + 2\log\left(\frac{8}{\delta}\right) \right)$.

\State Sample $\mathsf{S} \sim \pN(0, \brI_{k\times n})$. 

\If{$\lambda_{\min} > \gamma + \mathsf{z} + \frac{4(\textsc{C}^{2}_X + \textsc{C}^{2}_Y)\log(\nicefrac{1}{\delta})}{\varepsilon}$\ for $\sz \sim \mathrm{Lap}\left(\frac{4(\textsc{C}^{2}_X + \textsc{C}^{2}_Y)}{\eps}\right)$}
    \State \Return $\widetilde{\theta} \gets \left((\mathsf{S}X)^{\top} (\mathsf{S}X)\right)^{-1} (\mathsf{S}X)^{\top} (\mathsf{S}Y)$
\Else
    \State Sample noises $\xi_1 \sim \pN(0, \brI_{k \times d})$, $\xi_2 \sim \pN(0, \brI_k)$.
    \State \Return $\widetilde{\theta} \gets \left((\mathsf{S}X + \gamma\xi_1)^{\top} (\mathsf{S}X + \gamma\xi_1)\right)^{-1} (\mathsf{S}X + \gamma\xi_1)^{\top} (\mathsf{S}Y + \gamma\xi_2)$
\EndIf

\end{algorithmic}
\end{algorithm}

%%%%%%%%%%%%%%%%%%%%%%%%%%%%%%%%%%%%%%%%%%%%%%%%%%%%%%%%%%%%%%%%%%%%%%%%%%%%%%%%%%
\begin{algorithm}[H]
\caption{Sheffet's Algorithm with Our Analysis} 
\begin{algorithmic}[1]

\Require Dataset $(X,Y)$; Privacy parameters $\varepsilon, \delta$; Bounds: $\underset{i \in [n]}{\max}\ \norm{x_i}^2 \leq \textsc{C}^{2}_X$, $\underset{i \in [n]}{\max}\ |y_i|^2 \leq \textsc{C}^{2}_Y$; Hyperparameter $k$.

\State Compute $\lambda_{\min} := \lambda_{\min}((X, Y)^\top (X, Y))$.

\State    Set $\gamma$ s.t. $\underset{1<\alpha<\gamma}{\min} \ \left\{\varphi(\alpha;k,\gamma) + \log\left(1-\frac{1}{\alpha}\right) - \frac{\log(\alpha\delta)}{\alpha - 1}\right\} \leq \eps/2$. 

\State Sample $\mathsf{S} \sim \pN(0, \brI_{k\times n})$. 

\If{$\lambda_{\min} > \gamma + \mathsf{z} + \frac{4(\textsc{C}^{2}_X + \textsc{C}^{2}_Y)\log(\nicefrac{1}{\delta})}{\varepsilon}$\ for $\sz \sim \mathrm{Lap}\left(\frac{4(\textsc{C}^{2}_X + \textsc{C}^{2}_Y)}{\eps}\right)$}
    \State \Return $\widetilde{\theta} \gets \left((\mathsf{S}X)^{\top} (\mathsf{S}X)\right)^{-1} (\mathsf{S}X)^{\top} (\mathsf{S}Y)$
\Else
    \State Sample noises $\xi_1 \sim \pN(0, \brI_{k \times d})$, $\xi_2 \sim \pN(0, \brI_k)$.
    \State \Return $\widetilde{\theta} \gets \left((\mathsf{S}X + \gamma\xi_1)^{\top} (\mathsf{S}X + \gamma\xi_1)\right)^{-1} (\mathsf{S}X + \gamma\xi_1)^{\top} (\mathsf{S}Y + \gamma\xi_2)$
\EndIf

\end{algorithmic}
\end{algorithm}
%%%%%%%%%%%%%%%%%%%%%%%%%%%%%%%%%%%%%%%%%%%%%%%%%%%%%%%%%%%%%%%%%%%%%%%%%%%%%%%%%%%%%%%%%%%%%% 
\section{Algorithms: Logistic Regression}
\label{app:algorithms_logistic}  
\subsection{Objective Pertubation}
\label{app:objective_perturb}

\begin{algorithm}[H]
\caption{Objective Perturbation ~ \citep{pmlr_kifer_objectiveperturb}}
\begin{algorithmic}[1]
\Require Dataset $(X,Y)$; privacy parameters $\eps$ and $\delta$; Bound $\norm{x_i} \leq \textsc{C}_X$ for all $i\in[n]$; 

\State Set $\sigma = \frac{\sqrt{4\eps + 8\log(\nicefrac{2}{\delta})}}{\eps}\textsc{C}_X$ and $\Delta = \frac{\textsc{C}^{2}_X}{2\eps}$.

\State Sample $\mathsf{b} \sim \pN(0, \sigma^2 \brI_d)$

\State \Return $\widetilde{\theta} \gets \underset{\theta}{\argmin} \ \left\{\sum^{n}_{i=1}-\frac{1}{n}\log\left(1 + \exp\left\{-y_i x^{\top}_i\theta\right\}\right) + \frac{\mathsf{b}^{\top}\theta}{n} + \frac{\Delta}{2n}\norm{\theta}^2\right\}$.

\end{algorithmic}
\end{algorithm}

%%%%%%%%%%%%%%%%%%%%%%%%%%%%%%%%%%%%%%%%%%%%%%%%%%%%%%%%%%%%%%%%%%%%%%%%%%%%%%%%%%%%
\section{Experimental Details}
\label{app:details}

All the experiments were run on an NVIDIA A100 GPU.  

%%%%%%%%%%%%%%%%%%%%%%%%%%%%%%%%%%%%%%%%%%%%%%%%%%%%%%%%%%%%%%%%%%%%%%%%%%%%%%%%%%%%
\subsection{Linear Regression}
\label{app:ExpsDetails_Linear}
For the linear regression experiments, we used four datasets. The first two are real-world popular regression datasets: the Tecator dataset~\citep{Tecator_bib} and the Communities and Crime dataset~\citep{redmond2002data}. We have used a random train-test split of 80\%/20\% for generating a train and a test set.  The other two are synthetic datasets where the responses were generated via the linear model \( y_i = x_i^{\top}\theta_0 + \sigma \xi_i \), with \( \theta_0 \) sampled as a unit vector uniformly from the \( (d-1) \)-dimensional sphere, \( \xi_i \sim \mathrm{Unif}(-1, 1) \), and \( \sigma = 0.1 \). 

In the first synthetic dataset (termed \emph{Gaussian dataset}), the parameters were \( n = 8192, d = 512 \), and the covariates were sampled as \( x_i \sim \pN(0, \mathsf{Q}\mathsf{Q}^{\top}) \), where \( \mathsf{Q} \in \mathbb{R}^{d \times q} \) is a random semi-orthogonal matrix with \( q = 4 \), ensuring the data lies on a 4-dimensional subspace. The matrix \( \mathsf{Q} \) was generated via QR decomposition of a random matrix with \iid\ standard Gaussian entries. The second synthetic dataset (termed the \emph{synthetic dataset}) was constructed as follows. First, we sampled latent covariates \( \widetilde{x}_i \sim \pN(0, \brI_2) \). Then, we generated final covariates using a two-layer neural network:  
\[
x_i = \phi(\mathsf{W}_2 \phi(\mathsf{W}_1 \widetilde{x}_i + \mathsf{b}_1) + \mathsf{b}_2),
\]
where \( \phi(\cdot) \) is the element-wise sigmoid function, \( \mathsf{W}_1 \sim \pN(0, \brI_{100 \times 2}) \), \( \mathsf{W}_2 \sim \pN(0, \brI_{d \times 100}) \), \( \mathsf{b}_1 \sim \pN(0, 10^{-6} \cdot \brI_{100}) \), and \( \mathsf{b}_2 \sim \pN(0, 10^{-6} \cdot \brI_d) \). In our experiments, we have fixed $d = 2^{9}$. For both synthetic datasets, the train and test sets were generated independently, using the same fixed \( \theta_0 \) but with independent covariates and additive noise.

In all cases, we normalized the training data so that the maximum \( \ell_2 \)-norm of any training sample was 1. The test data was scaled using the same normalization factor as the training data.

The baseline (non-private) estimator was computed as \( \widehat{\theta} = (X^{\top}X + \lambda \brI_d)^{-1}X^{\top}Y \) for $\lambda = 10^{-6}$, ensuring invertibility in all cases. We report the mean squared error (MSE) for the test set, computed as the squared error in predicting \( y_i \) via \( x_i^{\top} \widehat{\theta} \). All results are averaged over 250 independent trials, and we report both the empirical means and confidence intervals.

%%%%%%%%%%%%%%%%%%%%%%%%%%%%%%%%%%%%%%%%%%%%%%%%%%%%%%%%%%%%%%%%%%%%%%%%%%%%%%%%%%%%
\subsubsection{Algorithms}
Our algorithm was implemented as described in \algorithmref{alg:linear_mix}. The ADASSP algorithm was implemented based on \citet[Algorithm~2]{wang_adassp}, following the procedure detailed in \appendixref{app:AdaSSP}. Our second baseline, from \citet[Algorithm~1]{sheffet2017differentially}, was implemented according to the description in \appendixref{app:sheffet}. This implementation matches that of \citet[Algorithm~1]{sheffet2017differentially}, except for an adjustment to account for a factor of 2 in the parameter \( \gamma \), which arises due to using the zero-out neighboring definition rather than the replacement definition. In the variant of this baseline that incorporates our improved privacy analysis, we replaced the original noise calibration with bounds derived from \lemmaref{lem:Privacy_First}, translated to \((\varepsilon, \delta)\)-DP using the conversion provided in \propositionref{prop:Renyi_classical_translate} (see also \appendixref{app:sheffet}).

%%%%%%%%%%%%%%%%%%%%%%%%%%%%%%%%%%%%%%%%%%%%%%%%%%%%%%%%%%%%%%%%%%%%%%%%%%%%%%%%%%%%%%%
\subsection{Logistic Regression}
\label{app:ExpsDetails_Logistic}
In this set of experiments, we trained a logistic regression classifier for a binary classification task without applying any regularization. Our non-private baseline is the standard \texttt{LogisticRegression} solver from the \texttt{sklearn.linear\_model} library. The private baseline is based on the objective perturbation method (described in \appendixref{app:algorithms_logistic}), where the minimization is carried out using \texttt{torch.optim.LBFGS} with a maximum of 500 iterations and a tolerance of $10^{-6}$, following the setup of \citet{GuoCertifiedDataRemoval}.

For our method, we tuned the hyperparameter $k$ via a grid search over the values $k$ ranging from $1.5d$ to $7.5d$ with increments of $0.5d$, selecting the value that yielded the best performance, and where we verified that the optimal value is not on the edge of the grid. The privacy cost of this hyperparameter tuning was not included in our accounting.

We conducted experiments on the Fashion-MNIST \citep{xiao2017fashion} and the CIFAR100 \citep{CIFAR10_Krizhevsky} datasets, using the implementations provided in \texttt{torchvision.datasets}. From each dataset, we selected only the samples corresponding to classes $3$ and $8$, and relabeled them as ${-1, 1}$ to fit the binary classification setting. We used the standard PyTorch train/test splits and normalized the training data by the maximum $L_2$ norm across all training samples, ensuring that each training sample has a norm of at most 1. The same normalization factor was then applied to the test set. The train and test loaders were generated using \texttt{torch.utils.data.DataLoader} with shuffling enabled. In \appendixref{s:log_reg_additional} we present additional simulations with the CIFAR10 \citep{CIFAR10_Krizhevsky} and the MNIST \citep{MNIST_LECUN} datasets. 

The network architecture used is a compact convolutional neural network for RGB image classification. It consists of two convolutional layers with ReLU activations and max pooling, reducing the input to a 64-channel feature map of size 8×8. The flattened features are passed through a fully connected layer with 128 hidden units and ReLU, followed by a final linear layer that outputs class logits. In both of the experiments, we have first trained this network end-to-end using the DP-SGD primitive implemented in Opacus \citep{yousefpour2021opacus}, where we have set the clipping parameter to $4.0$, learning rate to $0.001$, the number of epochs to $20$, and the batch size to $500$. 

Performance metrics are averaged over 250 independent runs, reporting test accuracy along with confidence intervals. Runtime comparisons show the ratio of execution times for the largest simulated $\varepsilon$.

%%%%%%%%%%%%%%%%%%%%%%%%%%%%%%%%%%%%%%%%%%%%%%%%%%%%%%%%%%%%%%%%%%%%%%%%%%%%%%%%%%%%
\section{Additional Experiments}
\label{app:additional_exps}

\subsection{Linear Regression}
\label{s:lin_reg_additional}
We have simulated additional four datasets: the Boston housing dataset \citep{harrison1978hedonic} that contains $506$ measurements of $13$-dimensional features with the goal of predicting house prices in the Boston area, the Wine quality dataset \citep{cortez2009wine} which contains $1359$ measurements of $11$-dimensional features, with the goal of predicting wine quality, the Bike sharing dataset \citep{bikesharing2019} with the goal of predicting the count of rental bikes, and another artificial dataset that follows the same description as that of the Gaussian dataset but now with \iid \ features where the distribution of each entry is $\mathrm{Unif}([-1,1])$ and the entries are independent between each other. The additional results are presented in \figureref{fig:additional_linear}. 

\begin{figure}[htbp]
  \centering
  % Row 1
  \begin{subfigure}[t]{0.425\textwidth}
    \centering
    \includegraphics[width=\linewidth]{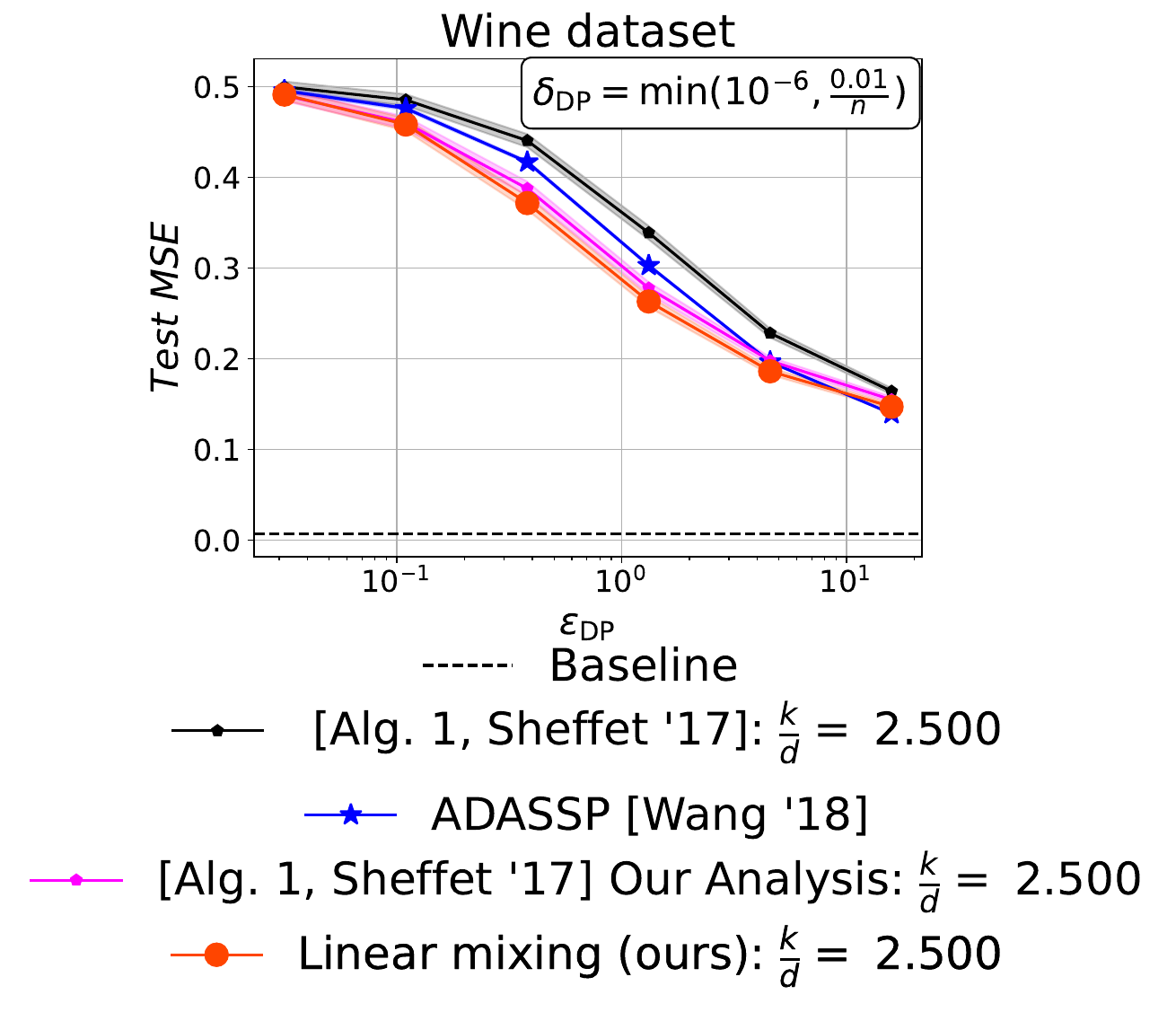}
    \caption{Wine dataset}
  \end{subfigure}
  \hspace{0.02\textwidth}
  \begin{subfigure}[t]{0.425\textwidth}
    \centering
    \includegraphics[width=\linewidth]{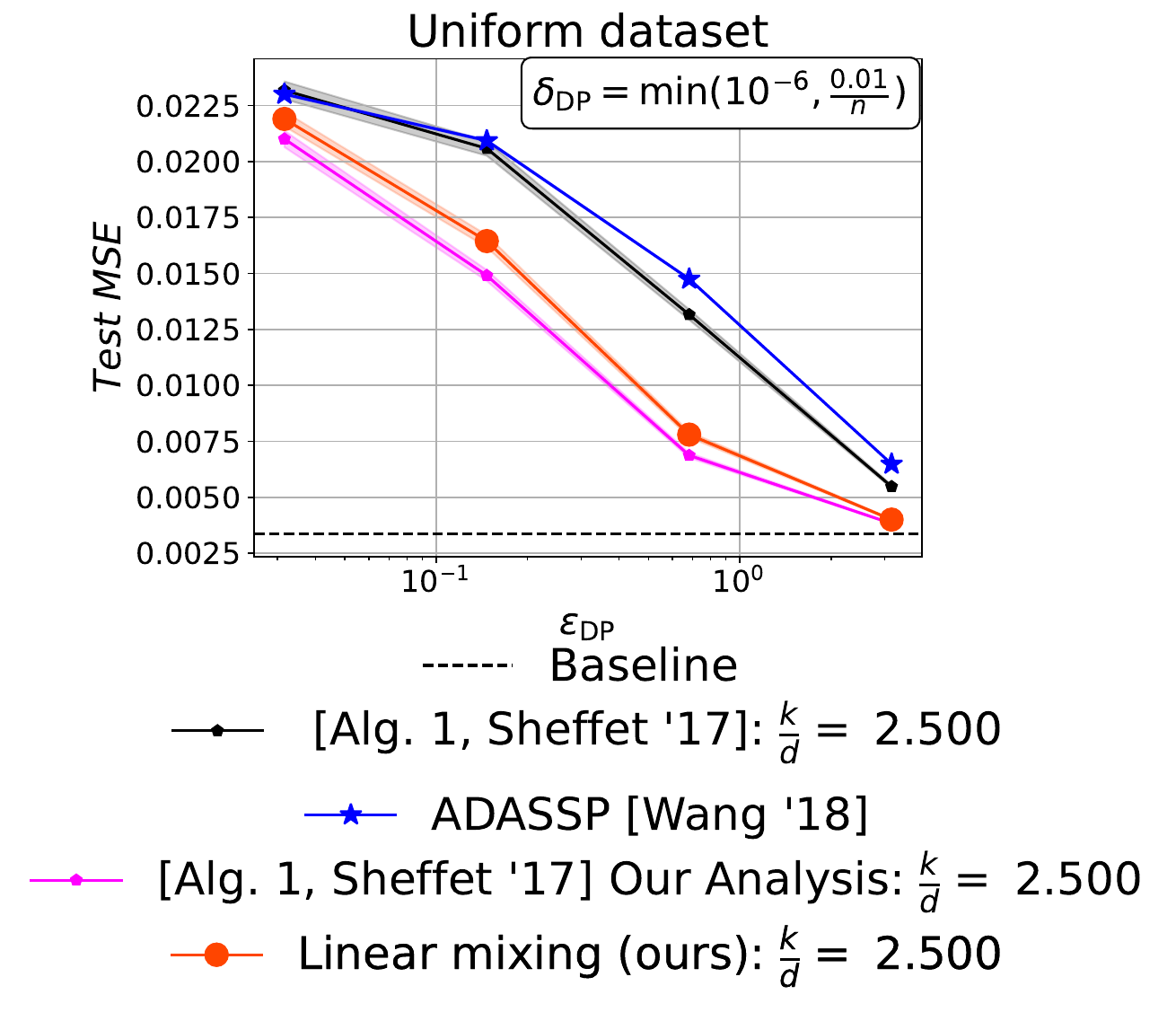}
    \caption{Uniform dataset}
  \end{subfigure}

  \vspace{0.5em} % Reduced vertical spacing between rows

  % Row 2
  \begin{subfigure}[t]{0.425\textwidth}
    \centering
    \includegraphics[width=\linewidth]{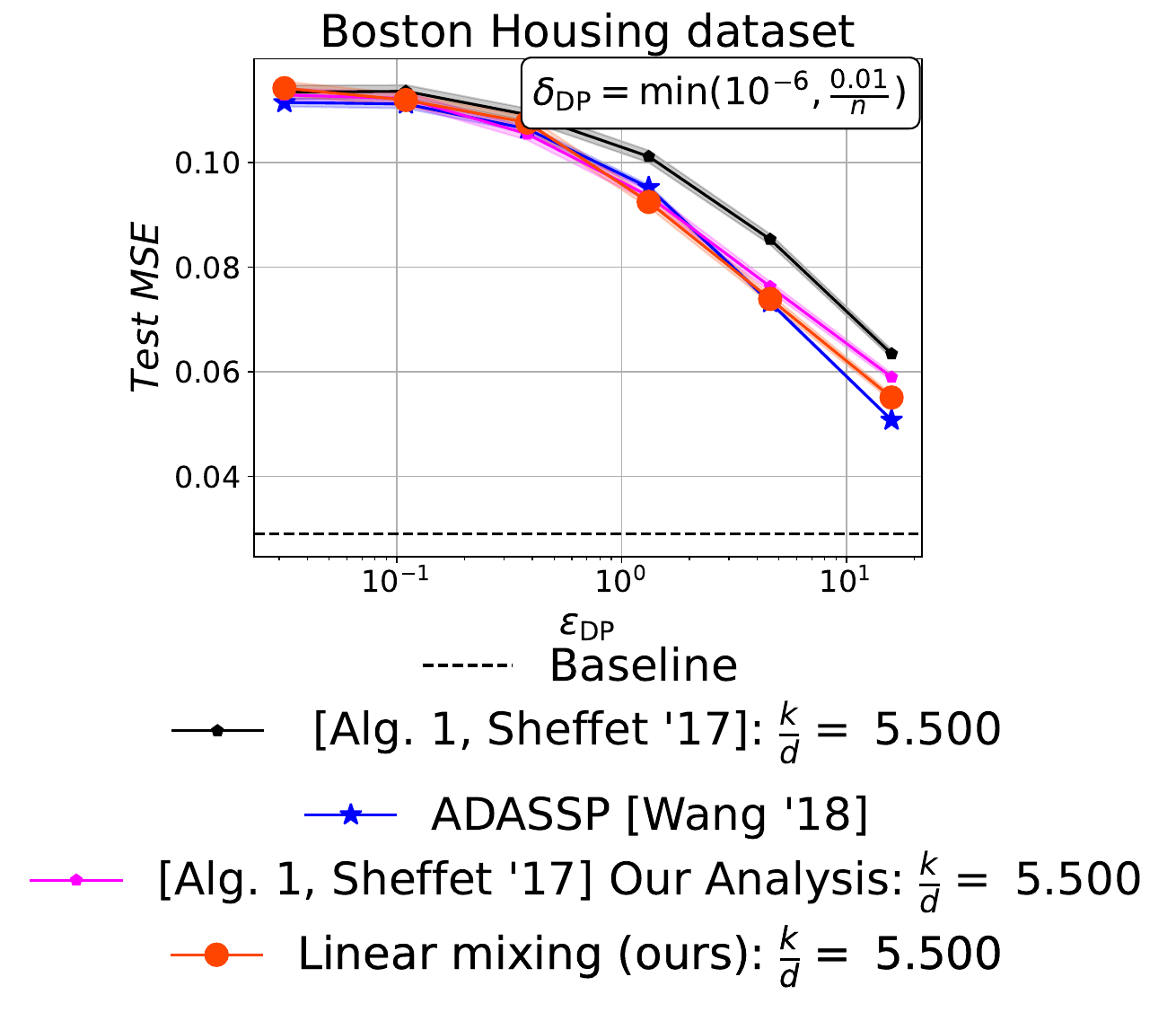}
    \caption{Boston housing dataset}
  \end{subfigure}
  \hspace{0.02\textwidth}
  \begin{subfigure}[t]{0.425\textwidth}
    \centering
    \includegraphics[width=\linewidth]{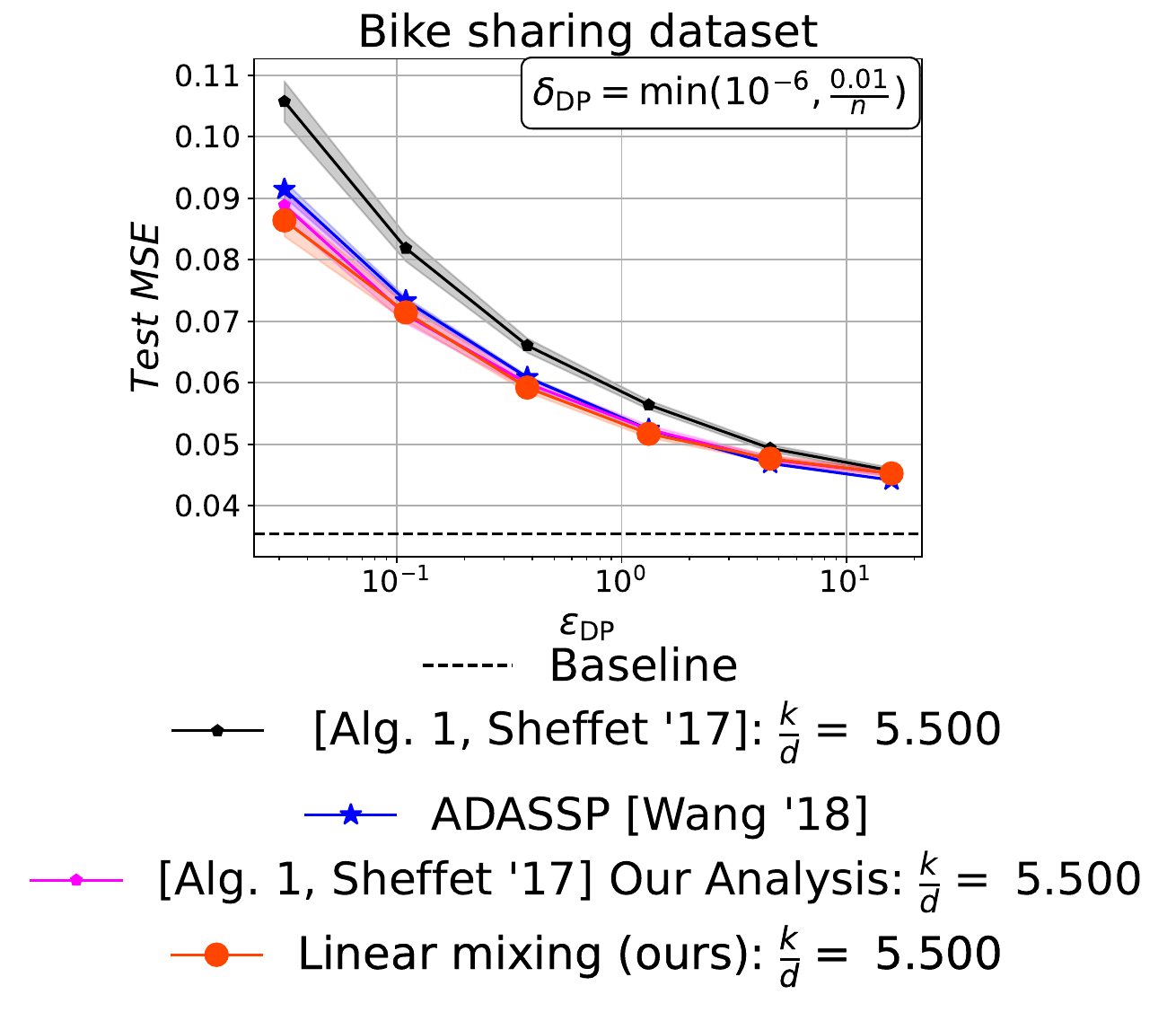}
    \caption{Bike sharing dataset}
  \end{subfigure}

  \caption{Linear mixing performance on the additional four linear regression tasks. }
  \label{fig:additional_linear}
\end{figure}

%%%%%%%%%%%%%%%%%%%%%%%%%%%%%%%%%%%%%%%%%%%%%%%%%%%%%%%%%%%
\subsection{Logistic Regression}
\label{s:log_reg_additional}

We have simulated two additional datasets: the CIFAR10 \citep{CIFAR10_Krizhevsky} and the MNIST \citep{MNIST_LECUN} datasets, using the same logistic regression setting. The datasets were used via similar procedure as described in \appendixref{app:ExpsDetails_Logistic}. The additional results are presented in \figureref{fig:additional_log}.

\begin{figure}[htbp]
  \centering
  \begin{subfigure}[b]{0.4\linewidth}
    \includegraphics[width=\linewidth]{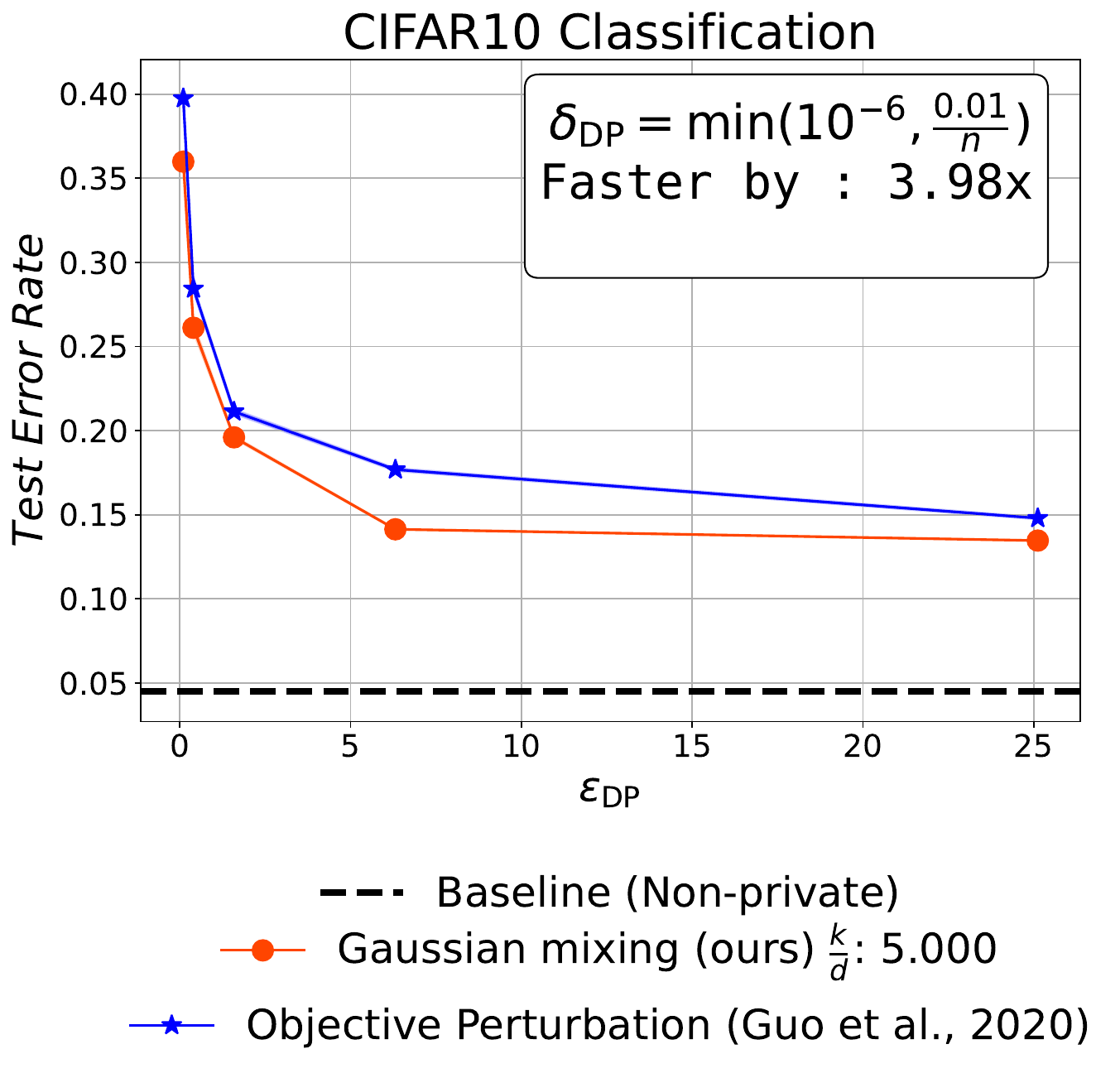}
    \caption{CIFAR-10}
  \end{subfigure}
  \hfill
  \begin{subfigure}[b]{0.4\linewidth}
    \includegraphics[width=\linewidth]{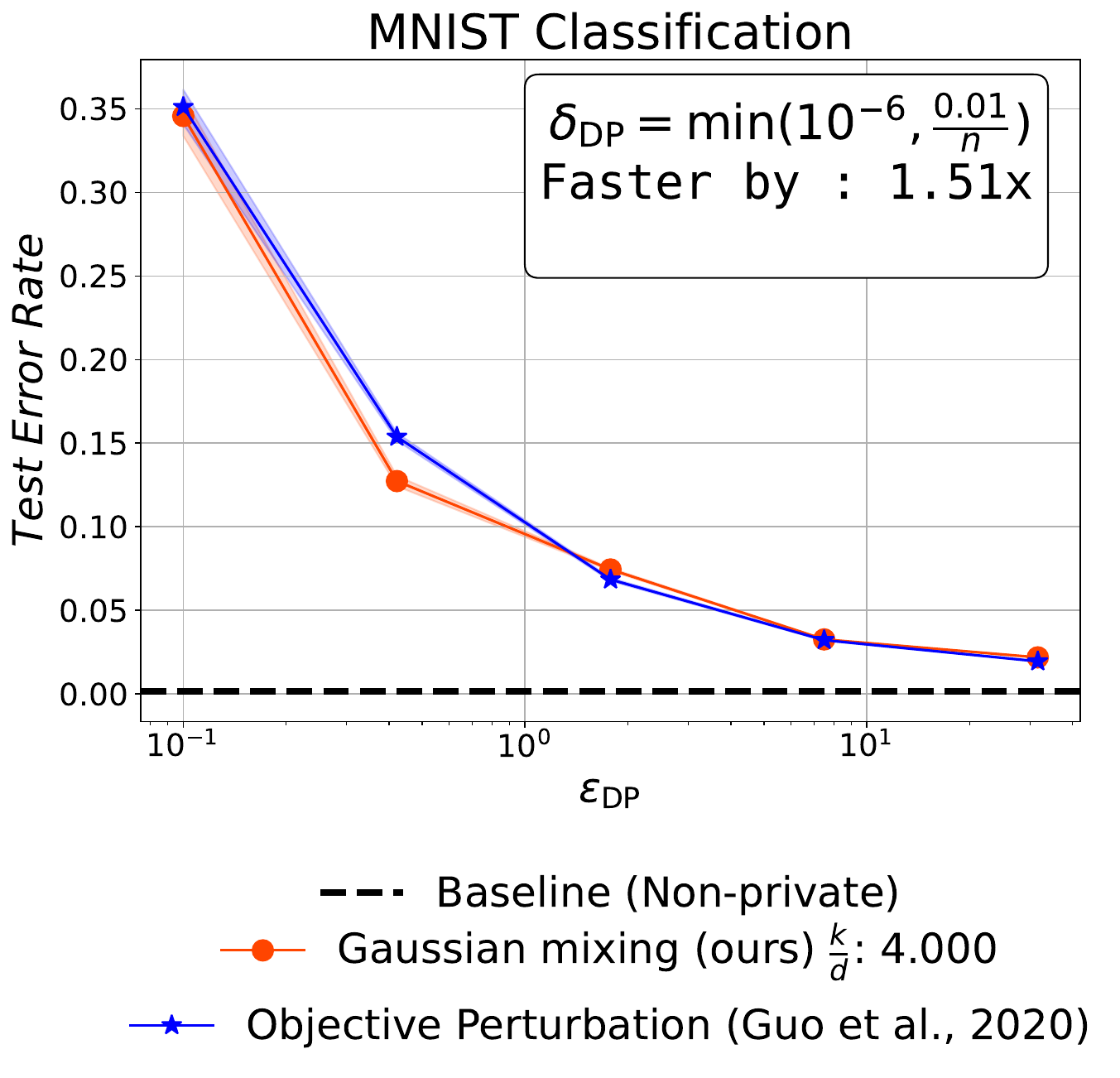}
    \caption{MNIST}
  \end{subfigure}
  
  \caption{DP logistic regression using a privately trained CNN feature extractor on binary subsets of CIFAR10 and MNIST. }
  \label{fig:additional_log}
\end{figure}
\end{document}

%% file: Config/defs.tex
\usepackage{amssymb}  % get, among others, blackboard bold fonts
\usepackage{verbatim} % get comment environment, + new verbatim
\usepackage{amsxtra}  % get, efg, \accentedsymbol
\usepackage{amsfonts}
\usepackage{color}
\usepackage{bbm}

\usepackage{subcaption}

\usepackage{wasysym}

\usepackage{letltxmacro}

\def\DHLhksqrt#1#2{%
\setbox0=\hbox{$#1\sqrt{#2\,}$}\dimen0=\ht0
\advance\dimen0-0.2\ht0
\setbox2=\hbox{\vrule height\ht0 depth -\dimen0}%
{\box0\lower0.4pt\box2}}

\usepackage{mathtools}
\mathtoolsset{showonlyrefs=true}
\newcommand{\isi}[1]{}
\newcommand{\isiincl}[2]{}

\newcommand{\googleincl}[2]{}
\newcommand{\googleinclabs}[3]{}

\usepackage{tikz}
\usetikzlibrary{shapes,arrows,decorations.markings,positioning}

\tikzstyle{plant} = [draw, fill=red!5, rectangle, 
    minimum height=3em, minimum width=6em]
\tikzstyle{block} = [draw, fill=blue!5, rectangle, 
    minimum height=3em, minimum width=6em]
\tikzstyle{sum} = [draw, fill=yellow!10, circle, node distance=1cm]
\tikzstyle{coord} = [coordinate]
\tikzstyle{gain} = [draw, fill=red!5, regular polygon, regular polygon sides=3, shape border rotate=-90]
\tikzstyle{pinstyle} = [pin edge={to-,thick,black}]

\tikzstyle{BitPipe} = [thick, decoration={markings,mark=at position
   1 with {\arrow[semithick]{open triangle 60}}},
   double distance=1.4pt, shorten >= 5.5pt,
   preaction = {decorate},
   postaction = {draw,line width=1.4pt, white,shorten >= 4.5pt}]
   
\usetikzlibrary{fit,backgrounds}

\usepackage{epsfig, graphicx, psfrag}

\usepackage[mathcal]{euscript}
\usepackage[cal=boondox]{mathalfa}

\DeclareMathAlphabet{\pazocal}{OMS}{zplm}{m}{n}

\usepackage{amsthm}

\newtheorem{thm}{Theorem}

\newtheorem{lem}{Lemma}

\newtheorem{corol}{Corollary}
\newtheorem{prop}{Proposition}

\theoremstyle{definition}
\newtheorem{defn}{Definition}
\newtheorem*{defn*}{Definition}

\newtheorem*{scheme*}{Scheme}

\theoremstyle{plain}

\providecommand{\theoremref}[1]{Theorem~\ref{#1}}
\providecommand{\sectionref}[1]{Section~\ref{#1}}
\providecommand{\corollaryref}[1]{Corollary~\ref{#1}}

\providecommand{\lemmaref}[1]{Lemma~\ref{#1}}
\providecommand{\appendixref}[1]{Appendix~\ref{#1}}

\providecommand{\figureref}[1]{Figure~\ref{#1}}
\providecommand{\definitionref}[1]
{Definition~\ref{#1}}
\providecommand{\algorithmref}[1]{Algorithm~\ref{#1}}
\providecommand{\propositionref}[1]{Proposition~\ref{#1}}

\newcommand{\ie}{i.e.}
\newcommand{\iid}{i.i.d.}

\newcommand{\reals}{\mathbb{R}}

\DeclareMathOperator{\R}{Re}

\newcommand{\old}[1]{}
\newcommand{\rem}[1]{}

\newcommand{\eps}{{\varepsilon}}

\providecommand{\brI}{\mathbb{I}}

\newcommand{\sx}{{\mathsf{x}}}

\newcommand{\sz}{{\mathsf{z}}}

\newcommand{\deter}[1]{\mathrm{det}\left(#1 \right)}

\newcommand{\abs}[1]{\left| #1 \right|}
\newcommand{\Norm}[1]{\left\| #1 \right\|}

\providecommand{\comment}[1]{}

\providecommand{\norm}[1]{\Norm{#1}}

\newcommand{\beqn}[1]{\begin{eqnarray}\label{#1}}
\newcommand{\eeqn}{\end{eqnarray}}
\newcommand{\beq}[1]{\begin{equation}\label{#1}}
\newcommand{\eeq}{\end{equation}}

\newcommand{\argmin}{\mathrm{argmin}}

\makeatletter
\newcommand{\vast}{\bBigg@{4}}
\newcommand{\Vast}{\bBigg@{5}}
\makeatother

\providecommand{\KLDA}[2]{{D_{\alpha} ( #1 \| #2 )}}

\providecommand{\bbE}{\mathbb{E}}

\providecommand{\E}[1]{\bbE \left[ #1 \right]}

\newcommand{\pK}{\pazocal{K}}

\newcommand{\pM}{\pazocal{M}}
\newcommand{\pN}{\pazocal{N}}

\newcommand{\pS}{\pazocal{S}}

\newcommand{\DP}{{\normalfont DP}~}
 
\newcommand{\RDP}{{\normalfont RDP}~} 
\newcommand{\Renyi}{{\normalfont R\'enyi}~} 
\newcommand{\RenyiDP}{{\normalfont \Renyi\hspace{-.3em}-\DP}~}